\documentclass{article}

\usepackage{microtype}
\usepackage{graphicx}
\usepackage{booktabs} 

\usepackage[hyphens]{url}

\usepackage[accepted]{icml2025}

\usepackage{hyperref}

\usepackage{amsmath}
\usepackage{amssymb}
\usepackage{mathtools}
\usepackage{amsthm}

\usepackage[capitalize,noabbrev]{cleveref}

\usepackage{tikz}
\usetikzlibrary{arrows.meta}
\usetikzlibrary{graphs.standard}
\usepackage{multirow} 
\usepackage{tabularx}
\usepackage{svg}
\usepackage{subcaption}
\usepackage{wrapfig}

\usepackage{amsmath,amsfonts,bm}
\usepackage{amssymb}

\newcommand{\Aut}{{\mathrm{Aut}}}
\newcommand{\Orb}{{\mathrm{Orb}}}
\newcommand{\Stab}{{\mathrm{Stab}}}

\newcommand{\confsub}[1]{\textsubscript{\scriptsize$\pm$#1}}

\newcommand{\appref}[1]{Appendix~\hyperref[#1]{\ref*{#1}}}
\newcommand{\figref}[1]{Figure~\hyperref[#1]{\ref*{#1}}}
\newcommand{\secref}[1]{Section~\hyperref[#1]{\ref*{#1}}}
\newcommand{\thmref}[1]{Theorem~\hyperref[#1]{\ref*{#1}}}
\newcommand{\lemref}[1]{Lemma~\hyperref[#1]{\ref*{#1}}}
\newcommand{\corref}[1]{Corollary~\hyperref[#1]{\ref*{#1}}}
\newcommand{\eqnref}[1]{Equation~(\hyperref[#1]{\ref*{#1}})}
\newcommand{\tabref}[1]{Table~\hyperref[#1]{\ref*{#1}}}

\def\1{\bm{1}}

\DeclareMathAlphabet{\mathsfit}{\encodingdefault}{\sfdefault}{m}{sl}
\SetMathAlphabet{\mathsfit}{bold}{\encodingdefault}{\sfdefault}{bx}{n}

\def\gA{{\mathcal{A}}}

\def\gE{{\mathcal{E}}}

\def\gG{{\mathcal{G}}}

\def\gL{{\mathcal{L}}}

\def\gO{{\mathcal{O}}}

\def\gS{{\mathcal{S}}}
\def\gT{{\mathcal{T}}}

\def\gX{{\mathcal{X}}}

\newcommand{\E}{\mathbb{E}}

\newcommand{\R}{\mathbb{R}}

\definecolor{color1}{RGB}{174, 182, 241}
\definecolor{color2}{RGB}{194, 221, 213}
\definecolor{color3}{RGB}{246, 245, 241}
\definecolor{color4}{RGB}{255, 160, 186}
\definecolor{color5}{RGB}{255, 210, 192}

\newcommand{\tikzcircle}[1][red]{\tikz\draw[fill=#1] (0,0) circle (.5ex);}

\theoremstyle{plain}
\newtheorem{theorem}{Theorem}[section]

\newtheorem{lemma}[theorem]{Lemma}
\newtheorem{corollary}[theorem]{Corollary}
\theoremstyle{definition}
\newtheorem{definition}[theorem]{Definition}

\theoremstyle{remark}

\usepackage[textsize=tiny]{todonotes}

\icmltitlerunning{Symmetry-Aware GFlowNets}

\begin{document}

\twocolumn[
\icmltitle{Symmetry-Aware GFlowNets}



\icmlsetsymbol{equal}{*}

\begin{icmlauthorlist}
\icmlauthor{Hohyun Kim}{sch}
\icmlauthor{Seunggeun Lee}{sch}
\icmlauthor{Min-hwan Oh}{sch}
\end{icmlauthorlist}

\icmlaffiliation{sch}{Graduate School of Data Science, Seoul National University, Seoul, Republic of Korea}

\icmlcorrespondingauthor{Seunggeun Lee}{lee7801@snu.ac.kr}
\icmlcorrespondingauthor{Min-hwan Oh}{minoh@snu.ac.kr}

\icmlkeywords{Machine Learning, ICML, GFlowNets, symmetry}

\vskip 0.3in
]



\printAffiliationsAndNotice{}  

\begin{abstract}
Generative Flow Networks (GFlowNets) offer a powerful framework for sampling graphs in proportion to their rewards. However, existing approaches suffer from systematic biases due to inaccuracies in state transition probability computations. These biases, rooted in the inherent symmetries of graphs, impact both atom-based and fragment-based generation schemes. To address this challenge, we introduce Symmetry-Aware GFlowNets (SA-GFN), a method that incorporates symmetry corrections into the learning process through reward scaling. By integrating bias correction directly into the reward structure, SA-GFN eliminates the need for explicit state transition computations. Empirical results show that SA-GFN enables unbiased sampling while enhancing diversity and consistently generating high-reward graphs that closely match the target distribution.
\end{abstract}

\section{Introduction}
GFlowNets have emerged as a powerful framework for learning generative models capable of sampling complex, compositional objects with probabilities proportional to a given reward. Inspired by reinforcement learning (RL), GFlowNets generate these objects through a sequence of actions that iteratively modify the structure of the object being built. This approach is particularly well-suited for generating compositional objects, such as graphs, where there are multiple paths for constructing an object. A prominent application of GFlowNets is molecule generation, where molecules are sequentially constructed as graphs \citep{bengio2021flow, jain2023gflownets}.

\begin{figure}[t]
\vskip 0.2in
\begin{center}
\centerline{\begin{tikzpicture}[
    roundnode/.style={circle, draw, minimum size=4mm, inner sep=0pt, font=\small}
]

\def\ACenterX{-4}
\def\ACenterY{0}

\def\BCenterX{1}
\def\BCenterY{1}

\def\CCenterX{0}
\def\CCenterY{-0.6}

\def\DCenterX{-1.4}
\def\DCenterY{-1.2}

\def\edgeLength{0.6}
\def\radius{60}
\def\coloralpha{40}

\node[roundnode, fill=yellow!\coloralpha] (ANode1) at (\ACenterX, \ACenterY) {$1$};
\node[roundnode, fill=yellow!\coloralpha] (ANode2) at (
    {\ACenterX - \edgeLength * cos(\radius)}, {\ACenterY + \edgeLength * sin(\radius)}
    ) {$2$};
\node[roundnode, fill=yellow!\coloralpha] (ANode3) at (
    {\ACenterX - \edgeLength * cos(\radius)}, {\ACenterY - \edgeLength * sin(\radius)}
    ) {$3$};
\node at (\ACenterX+0.4, \ACenterY-0.5) {$G_1$};
\node (AArrowAnchor) at (\ACenterX+0.4, \ACenterY+0.4) {};

\draw (ANode1) -- (ANode2);
\draw (ANode1) -- (ANode3);

\node[roundnode, fill=yellow!\coloralpha] (BNode1) at (\BCenterX, \BCenterY) {$1$};
\node[roundnode, fill=yellow!\coloralpha] (BNode2) at (
    {\BCenterX - \edgeLength * cos(\radius)}, {\BCenterY + \edgeLength * sin(\radius)}
    ) {$2$};
\node[roundnode, fill=yellow!\coloralpha] (BNode3) at (
    {\BCenterX - \edgeLength * cos(\radius)}, {\BCenterY - \edgeLength * sin(\radius)}
    ) {$3$};
\node[roundnode, fill=red!\coloralpha] (BNode4) at (
    {\BCenterX + \edgeLength}, {\BCenterY}
    ) {$4$};
\node at (\BCenterX+0.4, \BCenterY-0.5) {$G_4$};
\node (BArrowAnchor) at (\BCenterX-0.6, \BCenterY-0.3) {};

\draw (BNode1) -- (BNode2);
\draw (BNode1) -- (BNode3);
\draw (BNode1) -- (BNode4);

\node[roundnode, fill=yellow!\coloralpha] (CNode1) at (\CCenterX, \CCenterY) {$1$};
\node[roundnode, fill=yellow!\coloralpha] (CNode2) at (
    {\CCenterX - \edgeLength * cos(\radius)}, {\CCenterY + \edgeLength * sin(\radius)}
    ) {$2$};
\node[roundnode, fill=yellow!\coloralpha] (CNode3) at (
    {\CCenterX - \edgeLength * cos(\radius)}, {\CCenterY - \edgeLength * sin(\radius)}
    ) {$3$};
\node[roundnode, fill=red!\coloralpha] (CNode4) at (
    {\CCenterX - \edgeLength - \edgeLength * cos(\radius)}, {\CCenterY + \edgeLength * sin(\radius)}
    ) {$4$};
\node at (\CCenterX+0.4, \CCenterY-0.5) {$G_3$};
\node (CArrowAnchor) at (\CCenterX-1.2, \CCenterY+0.4) {};

\draw (CNode1) -- (CNode2);
\draw (CNode1) -- (CNode3);
\draw (CNode2) -- (CNode4);

\node[roundnode, fill=yellow!\coloralpha] (DNode1) at (\DCenterX, \DCenterY) {$1$};
\node[roundnode, fill=yellow!\coloralpha] (DNode2) at (
    {\DCenterX - \edgeLength * cos(\radius)}, {\DCenterY + \edgeLength * sin(\radius)}
    ) {$2$};
\node[roundnode, fill=yellow!\coloralpha] (DNode3) at (
    {\DCenterX - \edgeLength * cos(\radius)}, {\DCenterY - \edgeLength * sin(\radius)}
    ) {$3$};
\node[roundnode, fill=red!\coloralpha] (DNode4) at (
    {\DCenterX - \edgeLength - \edgeLength * cos(\radius)}, {\DCenterY - \edgeLength * sin(\radius)}
    ) {$4$};
\node at (\DCenterX+0.4, \DCenterY-0.5) {$G_2$};
\node (DArrowAnchor) at (\DCenterX-0.7, \DCenterY+0.1) {};

\draw (DNode1) -- (DNode2);
\draw (DNode1) -- (DNode3);
\draw (DNode3) -- (DNode4);

\draw[blue!40,thick] (\DCenterX+0.4, \DCenterY+0.1) ellipse (20mm and 14mm);

\draw[->, thick] (AArrowAnchor) to[out=-30, in=160, looseness=0.8] (DArrowAnchor);
\draw[->, thick] (AArrowAnchor) to[out=0, in=180, looseness=0.8] (CArrowAnchor);
\draw[->, thick] (AArrowAnchor) to[out=30, in=180, looseness=0.8] (BArrowAnchor);

\end{tikzpicture}}
\caption{Illustration of graph transitions from $G_1$ to various successor graphs. The blue oval highlights graphs $G_2$ and $G_3$ are isomorphic.}
\label{fig:intro}
\end{center}
\vskip -0.4in
\end{figure}
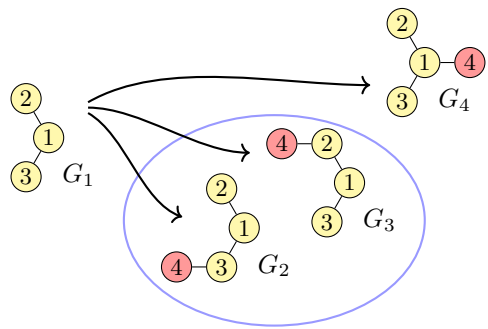

However, GFlowNet training objectives rely on the accurate computation of the transition probability of a policy, which becomes particularly challenging in graph-building environments due to the presence of \textit{equivalent actions}. These are actions that, although different, lead to the same graph structure. For instance, consider \figref{fig:intro}, where connecting a new node (node 4) to either of two existing nodes (nodes 2 or 3) results in isomorphic graphs. Although these actions are distinct, they lead to structurally identical graphs, meaning their transition probabilities must be summed. More generally, when multiple actions lead to the same state, the transition probability must account for all equivalent actions. This issue, referred to as the \textit{equivalent action problem}, arises because determining whether two actions result in the same state requires computationally expensive graph isomorphism tests \citep{ma2024baking}.

While GFlowNets were initially popularized for their reward-matching capabilities, experiments conducted by \citet{ma2024baking} demonstrate that neglecting to account for equivalent actions can introduce bias into GFlowNets. Our analysis further shows that the bias is systematic: it skews the model toward sampling graphs with fewer symmetries in node-by-node generation while favoring symmetric components in fragment-based generation. This bias is particularly problematic for tasks such as molecular generation, as molecules inherently possess natural symmetries. For instance, in the ZINC250k dataset, over 50\% of molecules exhibit more than one symmetry, with 18\% containing four or more symmetries. Ignoring symmetries results in incorrect modeling and inaccurate molecular structure generation, ultimately reducing the accuracy of the sampled molecules.

In this paper, we propose a simple yet effective modification to the GFlowNet training objectives to resolve the equivalent action problem. Our method adjusts the reward based on the number of symmetries in a graph, requiring only minimal changes to the existing training algorithms. Additionally, we introduce a new unbiased estimator for the model likelihood. Our key contributions are as follows:

\begin{itemize}
    \item We present a rigorous formulation of autoregressive graph generation within the GFlowNet framework, explicitly addressing the equivalent action problem. 
    \item We propose a simple yet effective method to address the equivalent action problem. Our approach scales the reward based on the automorphism group of the generated graph, enabling GFlowNets to accurately model and sample from the target distribution. Using a similar technique, we also derive an unbiased estimator for the model likelihood.
    \item Through experiments, we validate our theoretical results, and demonstrate the effectiveness of our method in generating diverse and high-reward samples.
\end{itemize}

\section{Related Work}

\paragraph{Autoregressive graph generation.} There are two primary formulations of autoregressive models: one based on adjacency matrices and the other based on graph sequences \citep{chen2021order}. Methods based on adjacency matrices \citep{you2018graphrnn, popova2019molecularrnn, liao2019efficient} are unlikely to suffer from the equivalent action problem because they preserve the node order information generated so far, making each pair of (graph, node order) a unique state. In contrast, equivalent actions arise in methods based on graph sequences \citep{you2018graph, li2018learning, shi2020graphaf}. This becomes problematic if a method requires state transition probabilities, as in GFlowNets. \citet{chen2021order} suggest that, for graph sequence-based methods, the size of a node’s orbit is related to the number of equivalent transitions, which inspired our work.

\paragraph{GFlowNets.} Several learning objectives have been proposed for GFlowNets, including flow matching \citep{bengio2021flow}, detailed balance \citep{bengio2023gflownet}, trajectory balance \citep{malkin2022trajectory}, sub-trajectory balance \citep{madan2023learning}, as well as their variants to improve training efficiency \citep{pan2023better, shen2023towards}. Recently, GFlowNets have been found to be equivalent to maximum entropy reinforcement learning \citep{tiapkin2024generative, mohammadpour2024maximum}, which was previously known to be inadequate for directed acyclic graph (DAG) environments \citep{bengio2021flow}. However, none of these objectives can avoid the equivalent action problem, as they are formalized based on state transitions, where multiple isomorphic graphs can represent the next state.

The work most closely related to ours is \citet{ma2024baking}, which highlighted the importance of accounting for equivalent actions to compute exact transition probabilities. Their approach involved an approximate test using positional encoding (PE) to detect equivalent actions at each transition. However, the bias in GFlowNets was validated only through experiments on synthetic datasets, without any theoretical analysis. Furthermore, their method relies on approximate tests that must be applied at every transition, making it computationally expensive. In contrast, our work provides an exact and efficient solution, requiring corrections only once at the end of trajectories rather than at each transition. This simplification not only reduces computational overhead but also makes our method straightforward to implement. Additionally, we present a comprehensive analysis showing that this bias arises in general settings, affecting both atom- and fragment-based generation schemes, and significantly impacts learning, particularly for highly symmetric graphs. Additional comparisons can be found in \appref{app:comparison}.

\section{Preliminaries}


\subsection{Graph Theory}

Let $G=(V, E)$ denote a graph, where $V = \{v_1, \dots, v_n\}$ is the set of $n$ vertices, and $E \subseteq V \times V$ is the set of edges. For heterogeneous graphs, we also define labeling functions $l_{n}$, $l_{e}$, and $l_{g}$, which map nodes, edges, and graphs to their respective attributes. We denote $\gG$ as the set of all such graphs under consideration. A permutation $\pi$ is a bijective mapping defined on the vertex set. We extend the permutation to sets as $\pi(V) = \{\pi(v) : v \in V\}$ and $\pi(E) = \{(\pi(v_i), \pi(v_j)) : (v_i, v_j) \in E\}$, as well as to the graph as $\pi(G) = (\pi(V), \pi(E))$. Since any permutation simply relabels node indices, it maps to a structurally identical graph. This notion is formalized as graph isomorphism.

\begin{definition}[Isomorphism]
    Two graphs $G = (V, E)$ and $G' = (V', E')$ are isomorphic, denoted $G \cong G'$, if there exists a permutation $\pi: V \rightarrow V'$ such that $\pi(E) = E'$. For heterogeneous graphs, the permutation must also preserve labels: for every $v \in V$, $l_n(v) = l_n'(\pi(v))$, for every $(u, v) \in E$, $l_e(u, v) = l_e'(\pi(u), \pi(v))$, and $l_g(G) = l_g'(G')$. 
\end{definition}

An automorphism is a special case of an isomorphism where the graph is mapped to itself.

\begin{definition}[Automorphism]
    An automorphism of a graph $G = (V, E)$ is a permutation $\pi$ on the vertex set $V$ that preserves the edge set, meaning $\pi(E) = E$. If labels are present, they must also be preserved under the permutation. The set of all automorphisms of a graph $G$ is called the automorphism group of $G$, denoted by $\Aut(G)$.
\end{definition}

In \figref{fig:intro}, graph $G_1$ has two automorphisms: the identity mapping and one that permutes nodes $2$ and $3$. We denote the order (or size) of the automorphism group as $|\Aut(G)|$, which represents the number of symmetries in the graph.

\begin{definition}[Orbit]
    The orbit of a node $u \in V$ in graph $G$ is defined as $\Orb(G, u) = \{v \in V : \exists \pi \in \Aut(G), \pi(u) = v\}$. Similarly, the orbit of an edge $(u, v) \in E$ in graph $G$ is defined as $\Orb(G, u, v) = \{(h, k) \in E : \exists \pi \in \Aut(G), (\pi(u), \pi(v)) = (h, k)\}$. More generally, the orbit of a node set $U \subseteq V$ in graph $G$ is defined as $\Orb(G, U) = \{U' : \exists \pi \in \Aut(G), \pi(U) = U'\}$.
\end{definition}

An orbit is a set of nodes or edges that are structurally identical. In \figref{fig:intro}, the orbit of node $2$ in graph $G_1$ is $\{2, 3\}$. Equivalent actions occur because they act on nodes in the same orbit; since nodes $2$ and $3$ are in the same orbit, adding a new node to either one is equivalent. This point will be further discussed in \secref{sec:the-equivalent-action-problem}.

\subsection{Generative Flow Networks}

The generation process of GFlowNets is defined by a finite directed acyclic graph (DAG) $(\gS, \gA)$, where $\gS = \{s_i\}$ is the set of states, and $\gA \subseteq \gS \times \gS$ is the set of state transitions. Let $s_0 \in \gS$ denote the special starting point of the process, called the initial state, with no incoming edges in the transition graph. Let $\gX \subseteq \gS$ be the set of terminal states, for which rewards are given. From the initial state $s_0$, objects are constructed sequentially by the forward transition policy $p_\gA(\cdot|s)$ until reaching terminal states. A set of complete trajectories, denoted as $\gT$, consists of sequences of transitions $\tau=(s_0, \dots, s_n)$ starting from the initial state $s_0$ and terminating at $s_n \in \gX$, such that $(s_t, s_{t+1}) \in \gA$. Let $\bar p_\gA(s)$ be the probability of reaching $s$ by following the policy $p_\gA$. The goal of GFlowNets is to train a policy $p_\gA$ that generates objects with a probability proportional to their reward. Specifically, the policy satisfies $\bar p_\gA(x) = R(x)/Z$ for all $x \in \gX$, where $Z$ is a normalizing constant. This is achieved by training $p_\gA$ using the following objectives.

\paragraph{Trajectory Balance \citep{malkin2022trajectory}.} The Trajectory Balance (TB) objective is based on the flow consistency constraint at the trajectory level. Given a complete trajectory $\tau$, the TB objective is defined as follows:

\vspace{-\baselineskip}
\begin{equation*}
    \vspace{0pt}
    \gL_{\mathrm{TB}}(\tau) 
        = \left( \log \frac{Z \prod_{t=0}^{n-1} p_\gA(s_{t+1}|s_t)}{R(s_n) \prod_{t=0}^{n-1} q_\gA(s_t|s_{t+1})} \right)^2.
\end{equation*}

It introduces a backward policy $q_\gA$ that reverses the process. Given $q_\gA$, the forward policy $p_\gA$ and the normalizing constant $Z$ are trained to match the backward flow induced by the reward function and the backward policy.

\paragraph{Detailed Balance \citep{bengio2023gflownet}.} The Detailed Balance (DB) objective is based on the flow consistency constraint at the state transition level. The objective is defined for each transition $(s, s')$ as:

\vspace{-\baselineskip}
\begin{equation*}
    \gL_{\mathrm{DB}}(s, s') 
        = \left( \log \frac{F(s) p_\gA(s'|s)}{F(s') q_\gA(s|s')} \right)^2.
\end{equation*}

The DB objective requires learning the state flow function $F: \gS \rightarrow \R^+$, which represents the unnormalized probability that the policy visits state $s$.

\section{The Equivalent Action Problem}
\label{sec:the-equivalent-action-problem}
In this section, we formalize the graph generation process in the context of GFlowNets and investigates its properties.

\subsection{Problem Definition}
\label{subsec:problem-definition}

Consider a sequential graph generation process $(\gG, \gE)$ that constructs graphs by modifying the nodes and edges of existing partial graphs, where $\gE \subseteq \gG \times \gG$ represents the set of transitions between graphs. In this section, we clarify the relationship between the two processes, $(\gS, \gA)$ and $(\gG, \gE)$.


Since isomorphism is an equivalence relation, it partitions the space $\gG$ into classes, where each graph in a class is structurally identical to the others. Let $[G] = \{G' \in \gG : G' \cong G\}$ denote the equivalence class of $G$ induced by graph isomorphism. The state space $\gS$ is defined as the set of equivalence classes of graphs, $\gS = \{[G] : G \in \gG \}$, rather than the graph space $\gG$ itself. This is because our goal in using GFlowNets is to sample \textit{any} graph within the equivalence class $s = [G]$ in proportion to $R(s)$. State transitions $\gA$ can also be defined by partitioning the graph transitions $\gE$ by the following equivalence relation:

\begin{definition}[Transition equivalence]
\label{dfn:transition-equivalence}
Graph transitions $(G_1, G_1')$ and $(G_2, G_2')$ are transition-equivalent if $G_1\cong G_2$ and $G_1'\cong G_2'$.
\end{definition}
\vspace{-0.5\baselineskip}

In practice, a graph generation process $(\gG, \gE)$ is constructed by first designing a set of allowable actions in a given graph. For example, \texttt{AddEdge}$(G, u, v)$ adds an edge $(u, v)$ to the existing graph $G$, and $\texttt{AddNode}(G, u)$ adds a new node to an existing node $u$. The $\texttt{Stop}(G)$ terminates the process, in which case the graph-level attribute is flagged as terminated. Through this construction, graph transitions $\gE$ are ensured to be structured so that any pair of isomorphic graphs have isomorphic successors.

In this setup, the state transition probability can be expressed in terms of graph transitions as follows: Let $p_\gE$ denote a forward policy over the graph space, and let $\gE(G) = \{G':(G, G')\in\gE\}$ denote the set of next graphs reachable from $G$. If $p_\gE$ is parameterized by permutation-equivariant network, then

\vspace{-\baselineskip}
\begin{equation}
\label{eqn:state transition probability}
    p_\gA(s'|s) = \sum_{G' \in \gE(G) \cap s'} p_\gE(G'|G),
\end{equation}
\vspace{-\baselineskip}

for any $G\in s$, where $\gE(G) \cap s'$ is the set of next graphs that are in the same equivalence class $s'$ (see \appref{app:state transition probability} for the derivation). This is an exact formula for computing state transition probabilities, which we use as a reference for comparison. However, it requires looking one step ahead and comparing the resulting graphs to identify $\gE(G) \cap s'$. This process involves multiple computationally expensive graph isomorphism tests for each transition. 

As an alternative, \citet{ma2024baking} suggested using orbits to identify transition-equivalent actions, which we formalize next. Define \textit{graph actions} as triples, $\bar\gE = \{(G, t, u)\}$, where $t$ denotes the action type and $u$ specifies the nodes or edges affected by the action. Since graph transitions $\gE$ are constructed by predefined set of graph actions, there is a one-to-one correspondence between $\gE$ and $\bar\gE$. However, we explicitly distinguish between them to introduce another equivalence relation.

\begin{definition}[Orbit equivalence]
\label{dfn:orbit-equivalence}
Graph actions $(G_1, t_1, u_1)$ and $(G_2, t_2, u_2)$ are orbit-equivalent if $t_1 = t_2$ and there exists a permutation $\pi$ such that $\pi(G_1)=G_2$ and $\pi(u_1)=u_2$. When $G_1 = G_2$, orbit equivalence indicates that $u_1$ and $u_2$ belong to the same orbit. 
\end{definition}

In \figref{fig:intro}, the transitions $(G_1, G_2)$ and $(G_1, G_3)$ are transition-equivalent because the resulting graphs are isomorphic, $G_2 \cong G_3$. These transitions are induced by the actions $\texttt{AddNode}(G_1, 2)$ and $\texttt{AddNode}(G_1, 3)$, which are orbit-equivalent since nodes $2$ and $3$ belong to the same orbit in graph $G_1$. Similarly to $\gA$, which is induced by transition equivalence relation, we define $\bar\gA$ as the set of equivalence classes induced by the orbit equivalence relation on $\bar\gE$. The notion of orbit equivalence is particularly useful because it serves as a replacement for transition equivalence. The next theorem establishes that orbit-equivalent actions induce equivalent transitions.

\begin{theorem}
    \label{thm:equivalence implication}
    Let $(G_1, t_1, u_1, G_1')$ and $(G_2, t_2, u_2, G_2')$ be two graph transitions induced by actions $e_1 = (G_1, t_1, u_1)$ and $e_2 = (G_2, t_2, u_2)$. If $e_1$ and $e_2$ are orbit-equivalent, then $(G_1, G_1')$ and $(G_2, G_2')$ are transition-equivalent.
\end{theorem}

In other words, graph actions operating on the same orbit lead to isomorphic graphs. This is because orbits, rather than individual nodes or edges, are structurally important in determining actions. The theorem implies that, in a given graph $G$, transition-equivalent actions---actions that lead to isomorphic graphs---can be identified by computing orbits. However, transition-equivalent actions are not always orbit-equivalent (see \appref{app:Example of Equivalent Actions}), meaning that the orbit equivalence relation provides a finer partition of $\gE$ than the transition equivalence. 

Given the distinction between transition-equivalent actions $\gA$ and orbit-equivalent actions $\bar\gA$, we define the state-action probability $p_{\bar\gA}(a|s)$ differently from the state transition probability $p_\gA(s'|s)$: while $p_\gA(s'|s)$ accounts for all transition-equivalent actions, $p_{\bar\gA}(a|s)$ aggregates only over orbit-equivalent actions. Formally, for a given graph $G$, let $\bar\gE(G)$ denote the set of graph actions available from $G$. Then, for a state-action pair $s \in \gS$ and $a \in \bar\gA$, the state-action probability is defined as:

\vspace{-\baselineskip}
\begin{equation}
\label{eqn:state action probability}
    p_{\bar\gA}(a|s) = \sum_{e \in \bar\gE(G) \cap a} p_{\bar \gE}(e|G)
\end{equation}
\vspace{-\baselineskip}

where $p_{\bar \gE}(e|G)$ denotes $p_\gE(G'|G)$ for $G'$ being the next graph. Here, $\bar\gE(G) \cap a$ represents the set of orbit-equivalent actions from $G$. By \thmref{thm:equivalence implication}, it follows that $p_{\bar\gA}(a|s) \leq p_\gA(s'|s)$ in general.

Computing exact state transition probabilities includes expensive graph isomorphism tests as stated. Instead, \citet{ma2024baking} proposed using \eqnref{eqn:state action probability}, observing that $p_{\bar\gA}(a|s)=p_\gA(s'|s)$ in most cases. However, computing orbits for every transition remains computationally intensive, which led them to develop approximate solutions. In the next section, we present an exact and efficient solution to address this challenge.

\subsection{Properties of Equivalent Actions}
\label{subsec:equivalent-actions}

\begin{figure*}[t]
\centering
    \begin{tikzpicture}[
    roundnode/.style={circle, draw, minimum size=3mm, inner sep=0pt}
]

\def\radius{0.6}

\def\angle{360/5}
\def\coloralpha{40}

\foreach \i in {1, 2, 3, 4, 5} {
    \coordinate (n\i) at ({\radius * cos(\angle*\i)}, {\radius * sin(\angle*\i)});
}


\foreach \i in {2, 3, 4, 5} {
    \node[roundnode] (g1n\i) at (n\i) {};
}

\node[roundnode, fill=color2] (g1n2) at (n2) {};
\node[roundnode, fill=color3] (g1n3) at (n3) {};
\node[roundnode, fill=color1] (g1n4) at (n4) {};
\node[roundnode, fill=color4] (g1n5) at (n5) {};

\draw (g1n2) -- (g1n3);
\draw (g1n3) -- (g1n4);
\draw (g1n4) -- (g1n5);

\def\length{\radius * 1.17557} 
\def\centerx{\radius*cos(216)}
\def\centery{\radius*sin(216)}
\def\nx{\centerx+\length*cos(198)}
\def\ny{\centery+\length*sin(198)}

\def\centeryarrow{\radius*sin(216) + 1}

\node[roundnode, fill=color1] (tail1) at ({\nx}, {\ny}) {};
\node[roundnode, fill=color4] (tail2) at (
    {\nx+\length*cos(108+18)}, {\ny+\length*sin(108+18)}) {};

\draw (g1n3) -- (tail1);
\draw (tail1) -- (tail2);

\node (g1t) at ({\centerx}, {\centery - 0.8}) {$G_1$};
\node[font=\small] at ({\centerx}, {\centery + 1.3}) {$|\Aut(G_1)| = 2$};

\node (g1p1r) at ({\centerx + 1.5}, {\centeryarrow - 0.6}) {};
\node (g1p2r) at ({\centerx + 1.5}, {\centeryarrow - 0.8}) {};
\node (g1p3r) at ({\centerx + 1.5}, {\centeryarrow - 1}) {};
\node (g1p4r) at ({\centerx + 1.5}, {\centeryarrow - 1.2}) {};

\begin{scope}[xshift=6.2cm]
\foreach \i in {1, 2, 3, 4, 5} {
    \coordinate (n\i) at ({\radius * cos(\angle*\i)}, {\radius * sin(\angle*\i)});
}

\node[roundnode, fill=color4] (g2n1) at (n1) {};
\node[roundnode, fill=color1] (g2n2) at (n2) {};
\node[roundnode, fill=color3] (g2n3) at (n3) {};
\node[roundnode, fill=color1] (g2n4) at (n4) {};
\node[roundnode, fill=color4] (g2n5) at (n5) {};


\draw (g2n1) -- (g2n2);
\draw (g2n2) -- (g2n3);
\draw (g2n3) -- (g2n4);
\draw (g2n4) -- (g2n5);

\node[roundnode, fill=color1] (tail1) at ({\nx}, {\ny}) {};
\node[roundnode, fill=color4] (tail2) at (
    {\nx+\length*cos(108+18)}, {\ny+\length*sin(108+18)}) {};
\draw (g2n3) -- (tail1);
\draw (tail1) -- (tail2);

\node (g2t) at ({\centerx}, {\centery - 0.8}) {$G_2$};
\node[font=\small] at ({\centerx}, {\centery + 1.3}) {$|\Aut(G_2)| = 6$};

\node (g2p1l) at ({\centerx - 1.4}, {\centeryarrow - 0.6}) {};
\node (g2p2l) at ({\centerx - 1.4}, {\centeryarrow - 0.8}) {};
\node (g2p3l) at ({\centerx - 1.4}, {\centeryarrow - 1}) {};
\node (g2p4l) at ({\centerx - 1.4}, {\centeryarrow - 1.2}) {};

\node (g2p1r) at ({\centerx + 1.4}, {\centeryarrow - 0.6}) {};
\node (g2p2r) at ({\centerx + 1.4}, {\centeryarrow - 0.8}) {};
\node (g2p3r) at ({\centerx + 1.4}, {\centeryarrow - 1}) {};
\node (g2p4r) at ({\centerx + 1.4}, {\centeryarrow - 1.2}) {};

\end{scope}


\begin{scope}[xshift=12.4cm]
\foreach \i in {1, 2, 3, 4, 5} {
    \coordinate (n\i) at ({\radius * cos(\angle*\i)}, {\radius * sin(\angle*\i)});
}

\node[roundnode, fill=color5] (g3n1) at (n1) {};
\node[roundnode, fill=color2] (g3n2) at (n2) {};
\node[roundnode, fill=color3] (g3n3) at (n3) {};
\node[roundnode, fill=color2] (g3n4) at (n4) {};
\node[roundnode, fill=color5] (g3n5) at (n5) {};

\draw (g3n1) -- (g3n2);
\draw (g3n2) -- (g3n3);
\draw (g3n3) -- (g3n4);
\draw (g3n4) -- (g3n5);
\draw (g3n5) -- (g3n1);

\node[roundnode, fill=color1] (tail1) at ({\nx}, {\ny}) {};
\node[roundnode, fill=color4] (tail2) at (
    {\nx+\length*cos(108+18)}, {\ny+\length*sin(108+18)}) {};

\draw (g3n3) -- (tail1);
\draw (tail1) -- (tail2);
\node (g3t) at ({\centerx}, {\centery - 0.8}) {$G_3$};
\node[font=\small] at ({\centerx}, {\centery + 1.3}) {$|\Aut(G_3)| = 2$};

\node (g3p1l) at ({\centerx - 1.4}, {\centeryarrow - 0.6}) {};
\node (g3p2l) at ({\centerx - 1.4}, {\centeryarrow - 0.8}) {};
\node (g3p3l) at ({\centerx - 1.4}, {\centeryarrow - 1}) {};
\node (g3p4l) at ({\centerx - 1.4}, {\centeryarrow - 1.2}) {};
\end{scope}

\def\yellowcirc{\tikz\draw[fill=color2] (0,0) circle (.7ex);}
\def\bluecirc{\tikz\draw[fill=color4] (0,0) circle (.7ex);}

\def\blueedge{
\tikz{
    \node[circle, draw, minimum size=1mm, inner sep=0pt, fill=color4] (a) {};
    \node[circle, draw, minimum size=1mm, inner sep=0pt, fill=color4, right=1.2mm] (b) {};
    \draw[-] (a) -- (b);
}
}
\def\magentaedge{
\tikz{
    \node[circle, draw, minimum size=1mm, inner sep=0pt, fill=color5] (a) {};
    \node[circle, draw, minimum size=1mm, inner sep=0pt, fill=color5, right=1.2mm] (b) {};
    \draw[-] (a) -- (b);
}
}

\draw[->, thick] (g1p1r) to[looseness=0.8] (g2p1l);
\draw[{Latex}-, thick] (g1p2r) [looseness=0]to (g2p2l);
\draw[{Latex}-, thick] (g1p3r) [looseness=0]to (g2p3l);
\draw[{Latex}-, thick] (g1p4r) [looseness=0]to (g2p4l);

\draw[->, thick] (g2p1r) to[looseness=0.8] (g3p1l);
\draw[->, thick] (g2p2r) to[looseness=0.6] (g3p2l);
\draw[->, thick] (g2p3r) to[looseness=0.4] (g3p3l);
\draw[{Latex}-, thick] (g2p4r) [looseness=0]to (g3p4l);

\end{tikzpicture}
    \vspace{-\baselineskip}
    \caption{Graphs representing transitions $(G_1, G_2, G_3)$, where the first transition is performed by \texttt{AddNode} and the second by \texttt{AddEdge}. The number of forward/backward actions are represented as the number of arrows. Symmetries in each graph is related to orbit-equivalent actions, as seen in the ratio $|\Aut(G_1)|/|\Aut(G_2)| = |\Orb(G_1,$ \tikzcircle[color2]$)|/|\Orb(G_2,$\tikzcircle[color4]$)|$. Nodes in the same orbit are given the same color.}
    \label{fig:transition}
    \vskip -0.2in
\end{figure*}
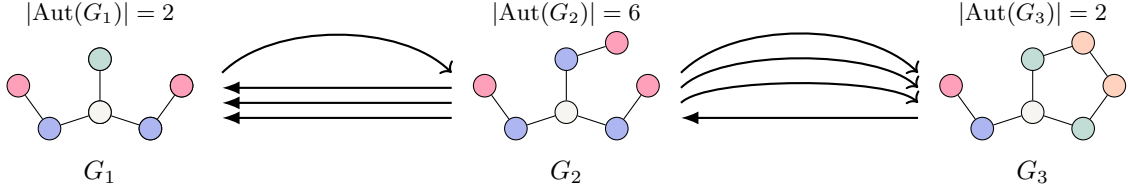

In fact, accounting for orbit equivalence is sufficient for GFlowNets, as demonstrated by the following theorem.

\begin{theorem}[Sufficiency of orbit equivalence]
    \label{thm:orbit equivalence sufficiency}
    Let $q_{\bar\gA}$ denote the backward state-action policy. State-action flow constraints are defined as $F(s)p_{\bar\gA}(a|s) = F(s')q_{\bar\gA}(a|s')$. If state-action flow constraints are satisfied for all possible state-action pairs, then the state transition flow constraints, $F(s)p_{\gA}(s'|s) = F(s')q_{\gA}(s|s')$, are also satisfied.
\end{theorem}

GFlowNets are typically formulated such that each pair of states are connected at most once in a DAG $(\gS, \gA)$, ensuring that every action leads to a unique next state. In contrast, in $(\gS, \bar\gA)$, distinct actions can lead to the same next state. This can be interpreted as multiple pathways connecting the same pair of states, enabling parallel flows. \thmref{thm:orbit equivalence sufficiency} essentially states that edge flows can be subdivided into multiple flows within a transition.

Our next goal is to simplify the computation of state-action probabilities. First note that $p_{\bar\gA}(a|s)$ can simply be expressed as:

\vspace{-\baselineskip}
\begin{align*}
    p_{\bar\gA}(a|s) = |\bar\gE(G) \cap a| \cdot p_\gE(G'|G)
\end{align*}
\vspace{-\baselineskip}

where $|\bar\gE(G) \cap a|$ represents the number of orbit-equivalent actions. This is because when actions are parameterized using graph neural networks (GNNs), orbit-equivalent actions are assigned equal probabilities. In general, permutation-equivariant functions produce identical representations for nodes within the same orbit (as detailed in \appref{app:permutation equivariance}). When node representations are aggregated to compute edge representations using invariant aggregators such as \texttt{SUM} or \texttt{MEAN}, edges within the same orbit also receive identical representations. Alternative parameterizations, such as the relative edge parameterization proposed by \citet{shen2023towards}, also assign equal probabilities to orbit-equivalent actions, while enhancing representational power. 

Note that the number of orbit-equivalent actions, $|\bar\gE(G) \cap a|$, corresponds to the size of the orbit of the associated nodes or edges. The following lemma shows that, when considering ratios, counting orbit-equivalent actions simplifies to counting automorphisms.

\begin{lemma}
    \label{lem:orb-aut}
    Let $G'=G[E\cup(u, v)]$ be a graph induced by adding an edge $(u, v)$ to graph $G$. Then, the following relationship holds:
    \begin{equation*}
        \frac{|\Orb(G, u, v)|}{|\Orb(G', u, v)|}= \frac{|\Aut(G)|}{|\Aut(G')|}
    \end{equation*}
    \vspace{-\baselineskip}
\end{lemma}

We presented \lemref{lem:orb-aut} in the context of \texttt{AddEdge} for simplicity, but similar lemma holds for other action types such as \texttt{AddNode}, \texttt{AddNodeAttributes}, and others used in node-by-node generation (See \appref{app:definitions of actions} for the full list of considered actions). We also consider \texttt{AddFragment} in our experiments, but discuss its properties separately in \appref{app:fragment-based}, as the corresponding formula differs in fragment-based schemes.

In \figref{fig:transition}, we observe that the number of equivalent actions changes as the graph evolves. For instance, from $G_1$, there is only one forward equivalent action, while from $G_2$, there are three. The number of backward actions also varies with each transition, making it seem daunting to account for all equivalent actions step-by-step. However, the ratio of forward and backward orbit-equivalent actions can be simply expressed as the ratio of the sizes of their automorphism groups. This is the basis for the next theorem.

\begin{theorem}[Automorphism correction]
    \label{thm:Automorphism correction}
    Let $q_\gE$ denote a graph-level policy defined for the backward process. Let $(G, G')$ be the graph transition, and $(s, a, s')$ denote a corresponding state transition. If permutation-equivariant functions are used for $p_\gE$ and $q_\gE$, then the following holds:
    \begin{equation*}
    \frac{p_{\bar\gA}(a|s)}{q_{\bar\gA}(a|s')} 
    = \frac{|\Aut(G)|}{|\Aut(G')|}\cdot\frac{p_\gE(G'|G)}{q_\gE(G|G')}.
    \end{equation*}
    \vspace{-\baselineskip}
\end{theorem}

The theorem suggests a simple adjustment method when considering the ratio. This simplification leads to the straightforward reward-scaling method presented in the next section.

\section{Symmetry-Aware GFlowNets}

\label{section:ac-gfn}
In this section, we analyze GFlowNet objectives using our previous results. The following theorem shows that a naive implementation of the TB objective, which does not account for equivalent actions, will train a model biased toward graphs with fewer symmetries.

\begin{corollary}[TB correction]
\label{cor:TB-correction}
Assume that $G_0$ is the empty graph or a single node, so that $|\Aut(G_0)| = 1$. Given the complete graph trajectory $\tau=(G_0, G_1, \dots, G_n)$, constructed using a node-by-node generation scheme, the trajectory balance loss is given by:

\vspace{-\baselineskip}
\begin{equation*}
    \gL_{\mathrm{TB}}(\tau)
        = \left( \log \frac{
        Z \prod_{t=0}^{n-1} p_\gE(G_{t+1}|G_t)}
        {|\Aut(G_n)| R(G_n) \prod_{t=0}^{n-1} q_\gE(G_t|G_{t+1})} 
        \right)^2.
\end{equation*}
\end{corollary}

The equation follows from \thmref{thm:Automorphism correction} and the application of a telescoping sum.

\paragraph{Implication.} \corref{cor:TB-correction} shows that we need to multiply the reward by the order of the automorphism group of the terminal state to properly account for equivalent actions. If we do not scale the reward, we are effectively reducing the rewards for highly symmetric graphs by a factor of $1/|\Aut(G_n)|$. As a result, even if a model is fully trained, the likelihood of reaching the terminal state will not align with the desired distribution; instead, the model is penalized for generating symmetric graphs, following $\bar p_\gA([G_n]) \propto R(G_n)/|\Aut(G_n)|$. This bias can be easily corrected by evaluating $|\Aut(G_n)|$ and scaling the reward accordingly.

We can also adjust the DB objective by multiplying the symmetry ratio $|\Aut(G')/|\Aut(G)|$ to the backward probability for each transition, though this requires multiple evaluations of automorphisms per trajectory. The next theorem states that, as in the TB correction, we can simply scale the rewards by $|\Aut(G)|$ without needing to count automorphisms at each transition. 

\begin{theorem}[DB correction]
    \label{thm:DB-correction}
    Consider a node-by-node graph generation scheme. We define the graph-level detailed balance condition, as opposed to the standard state-level condition, as follows:
    \begin{equation*}
    \tilde{F}(G)p_\gE(G'|G) = \tilde{F}(G')q_\gE(G|G'),
    \end{equation*}
    where $\tilde{F}$ denotes the graph-level flow function. If rewards are given by $\tilde{R}(G) =|\Aut(G)|R(G)$ and the graph-level detailed balance condition is satisfied for all transitions, then the forward policy samples terminal states proportionally to the given reward $R$.
\end{theorem}

\paragraph{Implication.} Together with \corref{cor:TB-correction}, we see that scaling the reward alone is sufficient for both TB and DB objectives. This suggests that other GFlowNet objectives, such as subtrajectory balance \citep{madan2023learning} and flow-matching \citep{bengio2021flow}, can also be used with reward scaling (see the discussion on the flow-matching \appref{app:Discussion on the Flow-Matching Objective}). This provides a straightforward approach to implementing GFlowNet objectives while reducing the computational burden of counting automorphisms at each transition. 

Finally, we provide the adjustment formula for fragment-based generation and defer the detailed discussion to \appref{app:fragment-based}.

\begin{theorem}[Fragment correction]
    \label{thm:Fragment correction}
    Let $G$ represents a terminal state ($[G] \in \gX$) generated by connecting $k$ fragments $\{C_1, \dots, C_k\}$. Then, the scaled rewards to offset the effects of equivalent actions are given by:
    \begin{align}
    \Tilde{R}(G) = \frac{|\Aut(G)|R(G)}{\prod_{i=1}^k |\Aut(C_i)|}    
    \label{eqn:frag-correction}
    \end{align}
    \vspace{-\baselineskip}
\end{theorem}

Intuitively, highly symmetric fragments contain many symmetric nodes available for connection, resulting in multiple forward equivalent actions, even though these actions do not lead to distinct outcomes. As a result, without correction, symmetric fragments are more likely to be sampled by the model. \eqnref{eqn:frag-correction} corrects this bias by penalizing symmetric fragments.


\paragraph{Estimating model likelihood.} To address the intractability of marginalizing over all trajectories terminating at $x \in \gX$, \citet{zhang2022generative} proposed approximating the model likelihood using importance sampling with $q_\gE$ as a variational distribution: $\bar p_\gA(x) = \E_{\tau \sim q_\gE(\tau|G_n)} \frac{p_\gE(\tau)}{q_\gE(\tau|G_n)}$, where $\tau = (G_0, \dots, G_n)$. However, \citet{zhang2022generative} worked with a restricted class of decision process where the equivalent action problem is not present. Instead, we estimate the probability of the terminal state as follows:

\vspace{-\baselineskip}
\begin{align}
\begin{split}
    \bar p_{\gA}(x) 
    &= \E_{\tau\sim q_\gE(\tau|G_n)} \left [ \frac{p_\gE(\tau)}{|\Aut(G_n)| q_\gE(\tau|G_n)} \right ] \\
    &\approx \frac{1}{M|\Aut(G_n)|}\sum_{i=1}^M \frac{p_\gE(\tau_i)}{ q_\gE(\tau_i|G_n)}.    
    \label{eqn:model-likelihood}
\end{split}
\end{align}
\vspace{-\baselineskip}

If we do not account for equivalent actions during both training and model likelihood estimation, the estimated model likelihood may still correlate with the rewards, but the actual sampling distribution will be biased. This happens because the policy is already biased towards generating samples with low $|\Aut(G_n)|$, leading to a spurious correlation.

\paragraph{Impact of GNN expressive power.}
Another source of inexact learning comes from the limited expressive power of GNNs used to parameterize the policy \citep{silva2025gflownets}. The correction formula relies on the parameter-sharing property of GNNs for nodes within the same orbit. While this property is desirable, actions from different orbits may also collapse into identical representations, thereby reducing the network's representational power. As a result, the policy might assign equal probabilities to actions that lead to different rewards. Although this issue is not the primary focus of this paper, we provide additional analysis of its impact in \appref{app:Expressive Power of GNNs}.

\paragraph{Computation.} The main additional computation for reward scaling comes from evaluating $|\Aut(G)|$, which is necessary for each trajectory in both the TB and DB objectives. For fragment correction, we can pre-compute $|\Aut(C)|$ in our vocabulary set. While the fastest proven time complexity for computing $|\Aut(G)|$ has remained $\exp(\gO(\sqrt{n\log n}))$ for decades \citep{babai1983computational}, graphs with bounded degrees can be handled in polynomial time \citep{luks1982isomorphism}. In our experiments, we used the \textit{bliss} algorithm \citep{junttila2007engineering}, included in the \texttt{igraph} package \citep{igraph}, and did not observe any significant delays in computation. In contrast, computing transition equivalent actions and summing their probabilities at each step involves several graph isomorphism tests. This process requires $K \times H$ more computations compared to the reward scaling, where $K$ is the average number of actions per state, and $H$ is the average trajectory length. We provide further analysis and comparisons on the computation time for each method in \appref{app:Computational Cost}.

\paragraph{Relation to distribution learning}
While our work focuses on improving the reward-matching capabilities of GFlowNets, many graph generation methods instead aim to learn the data distribution by maximizing a variational lower bound (VLB):

\vspace{-\baselineskip}
\begin{align}
\begin{split}
    \log \bar p_{\gA}(x) 
    \ge \E_{\tau\sim q(\tau|x)} \left [ \log p_\gA(\tau)-\log q(\tau|x) \right ].
    \label{eqn:VLB}
\end{split}
\end{align}
\vspace{-\baselineskip}

Here, $q(\tau|x)$ is a variational distribution that samples trajectories $\tau=(s_0, \dots, s_n=x)$ conditioned on the final state $x$. There are two primary ways to parameterize $q$ and sample trajectories $\tau$: (1) $q$ can be implemented as a backward policy $q(s_t|s_{t+1})$, as in GFlowNets; (2) $q(\pi|x)$ can first sample a node or edge ordering $\pi$, which then deterministically defines a trajectory $\tau$. Our method suggest that when the backward policy is used to define $q(s_t|s_{t+1})$, the VLB can be computed via automorphism counting. In contrast, when $q(\pi|x)$ is used to sample orderings, no further correction is required. This is because forward equivalent actions split probabilities across different orderings that induce the same trajectory. Since this result differs from that of \citet{chen2021order}, we elaborate on this point in \appref{appendix:node-ordering}, where we interpret equivalent actions as those that correspond to different orderings but ultimately lead to the same state sequence.

\begin{figure*}[t]
\begin{subfigure}{0.245\textwidth}
  \centering
  \includegraphics[width=\linewidth]{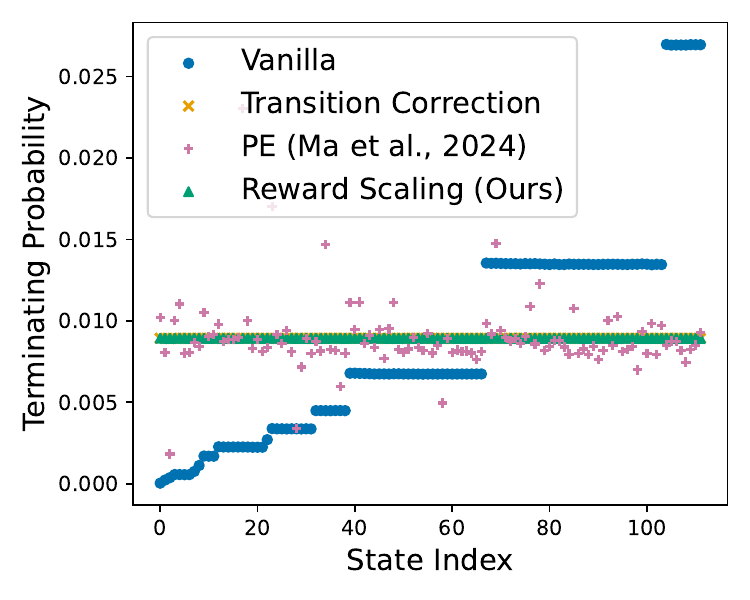}
  \caption{Illustrative}
  \label{fig:sfig1}
\end{subfigure}
\hfill
\begin{subfigure}{0.245\textwidth}
  \centering
  \includegraphics[width=\linewidth]{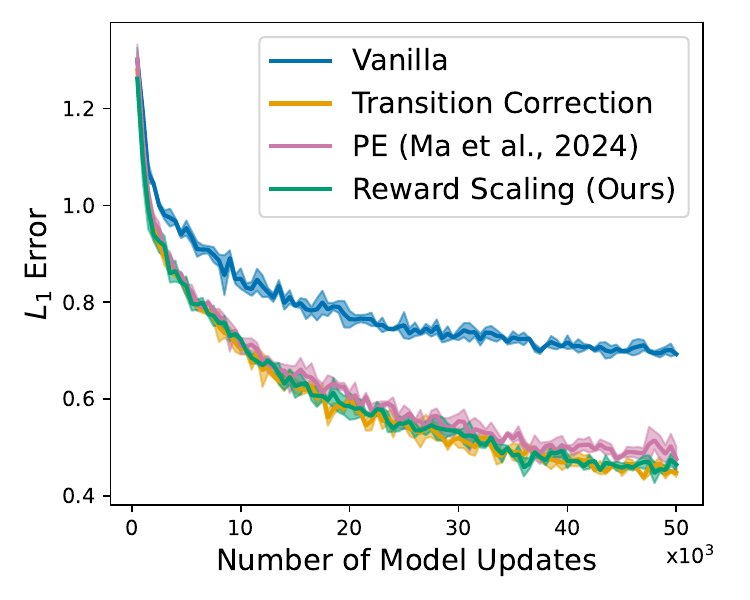}
  \caption{Synthetic (TB)}
  \label{fig:sfig2}
\end{subfigure}
\hfill
\begin{subfigure}{0.245\textwidth}
  \centering
  \includegraphics[width=\linewidth]{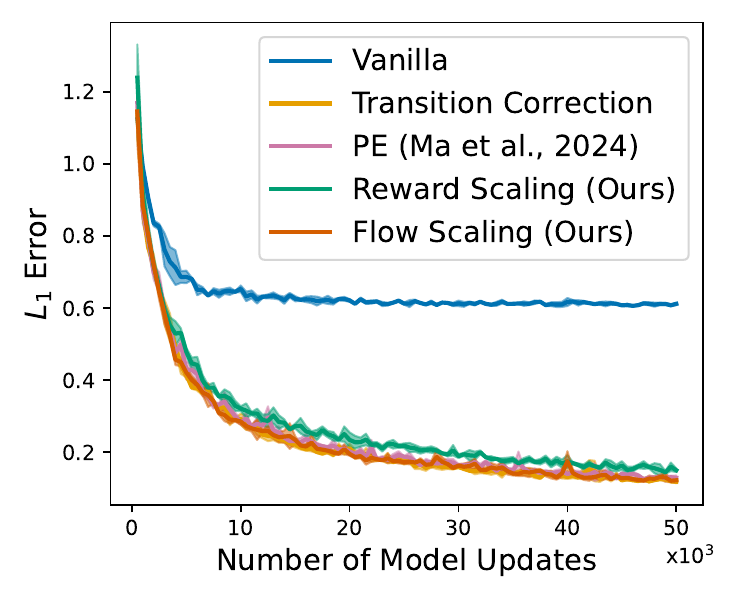}
  \caption{Synthetic (DB)}
  \label{fig:sfig3}
\end{subfigure}
\hfill
\begin{subfigure}{0.245\textwidth}
  \centering
  \includegraphics[width=\linewidth]{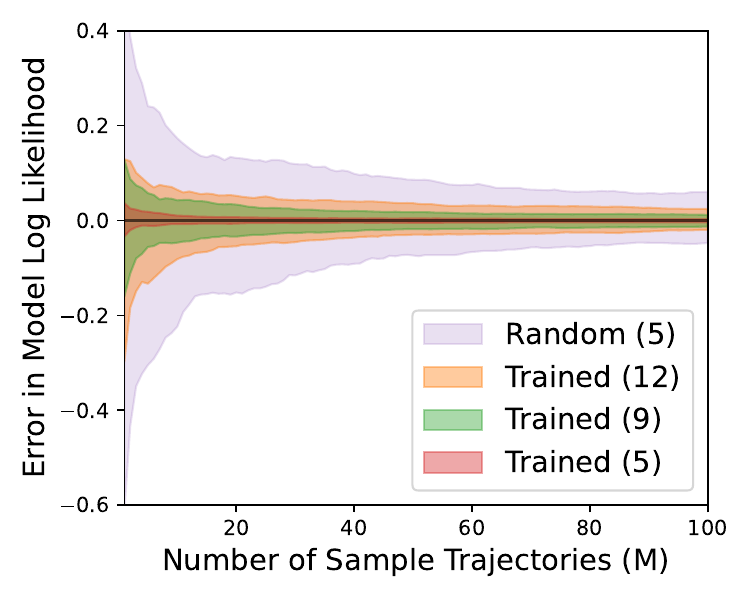}
  \caption{Likelihood estimation}
  \label{fig:sfig4}
\end{subfigure}
\caption{\textbf{(a)} Terminating probabilities of trained models in the uniform-reward environment. States are sorted according to the number of graphs in the state, $|x|$. \textbf{(b), (c)} $L_1$ errors between the target probabilities and the model’s terminating probabilities during training in the synthetic environment. \textbf{(d)} Errors in the estimated model log-likelihood, defined as the difference between estimated and exact log-likelihood. ``Random" denotes the errors of an initial random model, while ``Trained" refers that of a trained model. Numbers in brackets indicate the number of edges in the terminal states used for estimation.}
\label{fig:synthetic results}
\vspace{-10pt}
\end{figure*}

\section{Experiments}

In this section, we conduct experiments to validate our theoretical results and demonstrate the effectiveness of our method. We use a uniform backward policy across all experiments. Details on hyperparameters and model configurations can be found in \appref{app:experimental-details}. The experiments compare the following methods: (1) \textbf{Vanilla} GFlowNets, which do not incorporate graph symmetries. (2) \textbf{Transition Correction}, which identifies transition-equivalent actions by performing multiple isomorphism tests and sums their probabilities accordingly. (3) \textbf{PE}, the method proposed by \citet{ma2024baking}, which approximately identifies orbit-equivalent actions using positional encoding. (4) \textbf{Reward Scaling}, which achieves correction by modifying only the reward signals. (5) \textbf{Flow Scaling}, which multiplies symmetry ratio $|\Aut(G')|/|\Aut(G)|$ to backward probability at each transition. Note that methods 3-5 are orbit correction methods. Since Reward Scaling and Flow Scaling have the same effect under the TB objective, we only consider Flow Scaling when using the DB objective.\footnote{Source code available at: \url{https://github.com/hohyun312/sagfn}}

\subsection{Illustrative Example}

We first conducted an illustrative experiment where the initial state consisted of six disconnected nodes, and only \texttt{AddEdge} and \texttt{Stop} actions were allowed. Terminal states correspond to connected graphs, with a uniform reward of 1 assigned to each. \thmref{thm:Automorphism correction} predicts that the terminating probability of the vanilla GFlowNet will exhibit a bias proportional to $|\Aut(G_0)|/|\Aut(G_n)|$, where $|\Aut(G_0)|=6!$. This corresponds to the number of graphs isomorphic to the terminal state $x$, meaning $\bar p_\gA(x) \propto |x|$. Our method corrects this bias by multiplying the rewards by $1/|x|$.

We trained three policies using the TB objective and computed the exact terminating probabilities for all states ($|\gX| = 112$). As shown in \figref{fig:synthetic results} (a), the terminating probabilities of the vanilla model are clustered according to $|x|$. In contrast, when state transition probabilities $p_\gA(s'|s)$ are computed exactly by summing over transition-equivalent actions, the terminating probabilities are uniform as desired. Notably, Reward Scaling achieves the same effect, validating \thmref{thm:orbit equivalence sufficiency}. The PE method exhibits approximate bias correction, as indicated by the scatter in its results.

Upon inspecting the trained normalizing constant, we observed that with Reward Scaling, the estimated $Z$ is $112$, matching the true value. Without correction, however, $Z$ is trained to be significantly larger, reaching $26706$. This discrepancy arises because, in this example, the correction works by scaling down the rewards by $1/|x|$. Alternatively, since the $6!$ term in $|x|$ is constant, it can be absorbed into the normalizing constant $Z$. Therefore, the rewards could instead be scaled up by $|\Aut(x)|$ to achieve the same effect.

\subsection{Synthetic Graphs}

Following \citet{ma2024baking}, we set up a graph-building environment where nodes can be one of two types, and graphs can contain up to 7 nodes ($|\gX|=72296$). Rewards are assigned based on the number of 4-cliques that contain at least three nodes of the same type. For this relatively small environment, we compute exact terminating probabilities for all states without approximations for evaluation. 

The results for the TB objective are presented in \figref{fig:synthetic results} (b). The vanilla GFlowNet exhibits limited performance, as measured by $L_1$ errors between the target probabilities and the model’s terminating probabilities. In contrast, our method (Reward Scaling) has substantially lower errors, producing results similar to those obtained by explicitly computing state transition probabilities at each step (Transition Correction). Although the PE method is an approximate solution and underperforms compared to ours, it still significantly outperforms the vanilla baseline, underscoring the importance of applying a correction.

For the DB objective shown in \figref{fig:synthetic results} (c), we observe that Reward Scaling requires more training steps to converge than other correction methods. This is because the Reward Scaling corrects only at the end of trajectories, leaving intermediate probabilities remain inaccurate. This hinders training under the DB objective, which relies on intermediate probabilities. On the other hand, the per-transition correction can be interpreted as providing intermediate signals for the adjustment, similar to the idea of providing intermediate reward signals, as suggested by \citet{pan2023better}. Reward Scaling achieves the same goal, but defers the adjustment signal to the end of the trajectory. 

To evaluate the effectiveness of the proposed model likelihood estimator, we sampled 100 terminal states for each category (5, 9, and 12 edges), resulting in a total of 300 states, and estimated their model likelihood using \eqnref{eqn:model-likelihood}. The \figref{fig:synthetic results} (d) displays the estimation errors (computed as estimates minus exact values) for each category, with shaded bands representing one standard deviation. The estimates converge to the exact likelihood as $M$ increases, with notably rapid convergence for small values of $M$. 

However, the estimation error varies significantly depending on the task. For instance, terminal states with 5 edges can be estimated more accurately than those with 12 edges, as 12-edge states have substantially more trajectories leading to them in this environment, making the estimation problem more challenging. Additionally, a trained model's likelihood can be estimated more accurately than that of a random policy; in fact, for a fully trained model, a single sampled trajectory would be sufficient for an accurate estimation.

\subsection{Molecule Generation}

\begin{table*}[t]
\caption{Results for molecule generation task. Highest scores are highlighted.}
\label{tab:mol-results}
\begin{center}
\begin{tabular}{llccccc}
\hline
Task & Method & Diversity & Top $K$ div.  & Top $K$ reward & Div. Top $K$ & Uniq. Frac. \\
\hline
\multirow{2}{3.5em}{Atom}
    & Vanilla & 0.929\confsub{0.024} & \colorbox{orange!10}{0.077\confsub{0.022}} & 1.09\confsub{0.02} & 1.09\confsub{0.02} & 0.93\confsub{0.077} \\
    & Reward Scaling (Exact) & \colorbox{orange!10}{0.959\confsub{0.01}} & 0.046\confsub{0.006} & \colorbox{orange!10}{1.091\confsub{0.013}} & \colorbox{orange!10}{1.091\confsub{0.013}} & \colorbox{orange!10}{1.0\confsub{0.0}} \\
\hline
\multirow{3}{3.5em}{Fragment}
    & Vanilla & 0.877\confsub{0.001} & 0.153\confsub{0.003} & 0.941\confsub{0.002} & 0.941\confsub{0.002}  & 1.0\confsub{0.0}\\
    & Reward Scaling (Approx.) & \colorbox{orange!10}{0.88\confsub{0.001}} & \colorbox{orange!10}{0.164\confsub{0.008}} & 0.949\confsub{0.006} & 0.949\confsub{0.006} & 1.0\confsub{0.0} \\
    & Reward Scaling (Exact) & 0.879\confsub{0.0} & 0.151\confsub{0.002} & \colorbox{orange!10}{0.952\confsub{0.003}} & \colorbox{orange!10}{0.952\confsub{0.003}} & 1.0\confsub{0.0} \\
\hline
\end{tabular}
\end{center}
\vskip -0.2in
\end{table*}

\paragraph{Task description.} We investigate whether accurately modeling a given target distribution helps generate diverse and high-reward samples in practice. We examine the atom-based generation task from \citet{jain2023multi} and the fragment-based generation task from \citet{bengio2021flow}. In the atom-based task, the goal is to generate molecules by sequentially adding new atoms, edges, or setting their attributes. Rewards are provided by a proxy model, which predicts the HOMO-LUMO gap. In the fragment-based task, we use a predefined set of fragments, each with a predefined set of attachment points—nodes on the fragment where edges can connect. The task involves building a tree graph, where each node represents a fragment, and edges specify the attachment points on the two connected fragments. Rewards are determined by a proxy model that predicts the binding energy of a molecule to the sEH target. 

For the atom-based task, we simply scale the final rewards by the order of the automorphism group. For the fragment-based task, we additionally correct for fragment automorphisms as described in \eqnref{eqn:frag-correction}. We also explore an approximate correction scheme that offers computational benefits, as detailed in \appref{app:Approximate Correction Method}. We sampled 5,000 molecules from each method and evaluated them using common metrics. The definitions of these metrics are provided in \appref{app:experimental-details:Molecule Generation}.

\begin{figure}[th]
\begin{center}
\centerline{\includegraphics[width=0.4\textwidth]{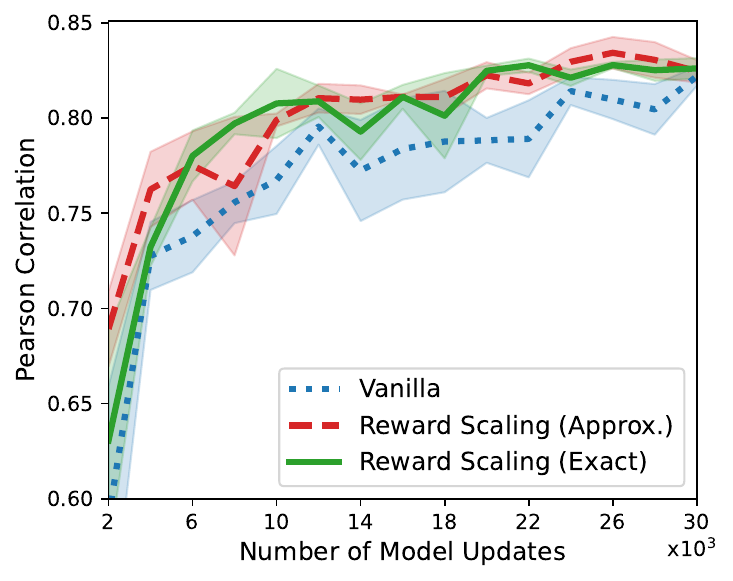}}
\caption{Correlations between log rewards and model log-likelihoods during training in fragment experiment.}
\label{fig:mol-pearson}
\end{center}
\vskip -0.42in
\end{figure}

\paragraph{Results.} The results summarized in \tabref{tab:mol-results} show that accurately modeling the target distribution yields generally the better results in terms of generating diverse and high-reward samples for both molecule tasks. This result is noteworthy, as the vanilla model effectively optimizes for two objectives---the proxy and non-symmetries---which could enhance diversity. However, atom-based task exhibited limited improvement, possibly due to the reward structure of the task, where rewards are negatively correlated with the number of symmetries when measured in the QM9 test set.

For the fragment-based task, the sampled molecules show higher rewards with our method. We also observe that the approximate correction already enables the generation of high-reward samples, underscoring the effectiveness and importance of the correction. Without correction, the trained model tends to excessively favor components that incur multiple forward equivalent actions during generation. For example, among 5000 sampled molecules, the vanilla GFlowNet produced 5220 instances of cyclohexane (\texttt{C1CCCCC1}) as its fragments, whereas the corrected method produced only 1042.

In addition, we measured the Pearson correlation between the estimated model log-likelihood, $\log \bar p_\gA(x)$, and the log rewards, $\log R(x)$, on the test set to validate the proposed fragment correction method. \figref{fig:mol-pearson} shows an overall high correlation for both exact and approximate corrections, emphasizing the impact of our methods.

\section{Discussion and Conclusion}

GFlowNets were first proposed as an alternative to previous methods, such as MaxEnt RL \citep{haarnoja2017reinforcement}, which are biased toward states with multiple action sequences leading to them. However, incorrect modeling of state transition probabilities introduces another type of bias in graph generation. Although we believe that the previous experimental results remain valid if interpreted carefully with the problem in mind, we recommend being explicit about the correction method used in all future work.

In this paper, we analyzed the properties of equivalent actions and proposed a simple correction method that allows for unbiased sampling from the target distribution. Our analysis shows that, without correction, highly symmetric graphs are less likely to be sampled, while symmetric fragments are more likely to be sampled, which is crucial for molecule discovery. We demonstrated that the reward-scaling technique works for both TB and DB objectives. Experimental results suggest that reward scaling and flow scaling effectively removes bias, allowing for accurate modeling of the target distribution, which is essential for sampling high-reward molecules. The exact effect, however, depends on the reward structure of the given task.

While our method is general and applicable to both node-by-node and fragment-based generation schemes, our theoretical guarantees rely on a specific set of predefined graph actions. Therefore, when designing a new set of graph actions, it is important to ensure that they share a similar structure, so that the theorems remain applicable. In most cases, however, the graph actions we introduced can be readily extended to incorporate additional actions. See \appref{app:definitions of actions} for further discussion.

A potential limitation of this paper is that the proposed correction method is demonstrated primarily on specific objectives (TB and DB) and datasets relevant to molecule discovery. Future work could explore applying the method to tasks with different symmetry patterns and reward structures.




\section*{Acknowledgements}
This work was supported by the National Research Foundation of Korea~(NRF) grant funded by the Korea government~(MSIT) (No.  RS-2022-NR071853 and RS-2023-00222663), the Global-LAMP Program of the NRF grant funded by the Ministry of Education (No. RS-2023-00301976), Brain Pool Plus~(BP+,~Brain Pool+) Program through the National Research Foundation of Korea~(NRF) funded by the Ministry of Science and ICT (2020H1D3A2A03100666), Korea Government Grant Program for Education and Research in Medical AI through the Korea Health Industry Development Institude~(KHIDI) funded by the Korea government~(MOE,~MOHW), and AI-Bio Research Grant through Seoul National University.

\section*{Impact Statement}

This paper presents work whose goal is to advance the field of 
Machine Learning. There are many potential societal consequences 
of our work, none which we feel must be specifically highlighted here.

\bibliography{references}
\bibliographystyle{icml2025}

\newpage
\appendix
\onecolumn
\section{Notations}
\label{app:notation}

\begin{table}[h]
\caption{Notation}
\label{tab:notations}
\begin{center}
\begin{tabular}{lll}
\hline
\multirow{18}{4em}{Graph} & 
    Set of graphs & $\gG$ \\
    & Set of graph transitions & $\gE$ \\
    & Set of graph actions & $\bar\gE$ \\
    & Set of vertices & $V$ \\
    & Set of edges & $E$ \\
    & Node labeling function & $l_n$ \\
    & Edge labeling function & $l_e$ \\
    & Graph labeling function & $l_g$ \\
    & Permutation & $\pi$ \\
    & Set of automorphisms of graph $G$ & $\Aut(G)$ \\
    & Set of graphs isomorphic to graph $G$ & $[G]$ \\
    & Orbit of a node $u$ & $\Orb(G, u)$ \\
    & Stabilizer of a node $u$ & $\Stab(G, u)$ \\
    & Forward policy over graphs & $p_\gE$ \\
    & Backward policy over graphs & $q_\gE$ \\
    & Set of next graphs from graph $G$ & $\gE(G)$ \\
    & Set of actions from graph $G$ & $\bar \gE(G)$ \\
    & Graph-action probability & $p_{\bar\gE}$ \\
    & Graph-level flow function & $\tilde F$ \\
\hline
\multirow{10}{4em}{State} & 
    Set of states & $\gS$ \\
    & Set of state transitions & $\gA$ \\
    & Set of actions & $\bar\gA$ \\
    & Set of terminal states & $\gX$ \\
    & Set of complete trajectories & $\gT$ \\
    & Reward function & $R$ \\
    & Forward policy over states & $p_\gA$ \\
    & Backward policy over states & $q_\gA$ \\
    & Terminating probability induced by following $p_\gA$ & $\bar p_\gA$ \\
    & State-action probability & $ p_{\bar\gA}$ \\
    & State-level flow function  & $F$ \\
\hline
\end{tabular}
\end{center}
\end{table}

\section{Additional Comparison to Prior Work}
\label{app:comparison}

To the best of our knowledge, \citet{ma2024baking} is the only prior work addressing the equivalent action problem in GFlowNets. Their approach relies on approximate tests using positional encoding (PE) of nodes to identify nodes or edges within the same orbit. Once an orbit is identified, the probabilities of orbit-equivalent actions are summed. While they identified and partially addressed this issue, their discussion was limited to experimental validation. The primary motivation of \citet{ma2024baking} was to highlight the existence of the problem and propose a partial solution. To this end, they conducted experiments in an offline, atom-based environment.

In contrast, our work provides the first rigorous theoretical foundation for the correction, demonstrating that this issue is not merely an experimental artifact but a fundamental and systematic challenge arising from graph symmetries in both atom-based and fragment-based generation. This insight is particularly significant given that GFlowNets were initially popularized for their reward-matching capabilities.

Our approach, based on reward/flow scaling, offers an exact and efficient solution. Unlike PE-based methods, which require adaptation for different action types (e.g., incorporating edge types and fragments), our method is straightforward to implement and easily generalizable across various action types. Our motivation is to thoroughly analyze the problem and present an efficient, scalable solution applicable to real-world setups. To validate this, we conducted experiments with online training for both atom- and fragment-based generation.

\section{Definitions of Graph Actions}
\label{app:definitions of actions}

Here we provide a list of action types considered in the paper.

\begin{itemize}
    \item \texttt{AddNode}$(G, u)$ adds a new node to the existing node $u$.
    \item \texttt{AddEdge}$(G, u, v)$ adds a new edge $(u, v)$.
    \item \texttt{AddFragment}$(G, C)$ adds a fragment $C$.
    \item \texttt{RemoveNode}$(G, v)$ removes node $v$ and its connecting edges.
    \item \texttt{RemoveEdge}$(G, u, v)$ removes the edge $(u, v)$.
    \item \texttt{RemoveFragment}$(G, C)$ removes the subgraph $C$.
    \item \texttt{SetNodeAttribute}$(G, u, t)$ sets the node-level attribute $t$ for node $u$.
    \item \texttt{SetEdgeAttribute}$(G, u, v, t)$ sets the edge-level attribute $t$ for the edge $(u, v)$.
    \item \texttt{SetGraphAttribute}$(G, t)$ sets a graph-level attribute $t$.
\end{itemize}

The \texttt{Stop} action can be interpreted as setting a terminal flag, making \texttt{SetGraphAttribute} a viable replacement. Some actions may overlap in functionality. For instance, \texttt{AddNode} is equivalent to a sequence of two actions: adding a new node and connecting it to an existing node $u$. This can be achieved using a combination of \texttt{AddFragment} and \texttt{AddEdge}.

Likewise, the above graph actions can be easily extended to incorporate additional actions. For example, we can define $\texttt{AddColoredNode}(G, u, t)$ as an action that adds a new node with node type $t$ to the existing node $u$. In fact, we used \texttt{AddColoredNode}, instead of using \texttt{AddNode} and \texttt{SetNodeAttribute} for the Synthetic Graphs experiment.

From a practical perspective, defining \texttt{AddColoredNode} as a separate action type reduces the number of transitions per trajectory and improves convenience. For theorems, however, proving \texttt{AddColoredNode} as a distinct case may be redundant.

\section{Example of Equivalent Actions}
\label{app:Example of Equivalent Actions}

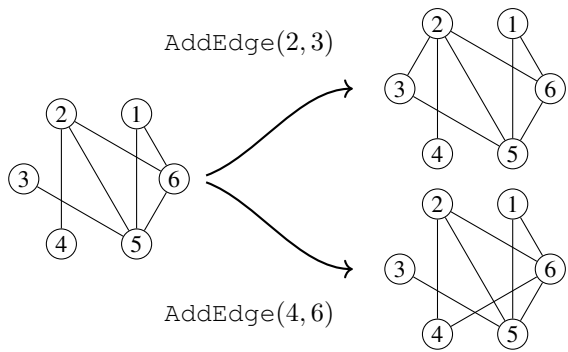
\begin{wrapfigure}{tr}{0.5\textwidth}
    \vspace{-30pt}
    \centering
    \begin{tikzpicture}[
    roundnode/.style={circle, draw, minimum size=4mm, inner sep=0pt, font=\small}
]

    \def\ACenterX{1.2}
    \def\radius{1.0}

    \def\angle{360/6}
    
    \foreach \i in {1, 2, 3, 4, 5, 6} {
        \node[roundnode] (n\i) at ({\ACenterX + \radius * cos(\angle*\i)}, {\radius * sin(\angle*\i)}) {\i};
    }
    
    \draw (n1) -- (n6);
    \draw (n1) -- (n5);
    \draw (n2) -- (n4);
    \draw (n2) -- (n5);
    \draw (n2) -- (n6);
    \draw (n3) -- (n5);
    \draw (n5) -- (n6);

    \node (nv) at ({0.3 + \ACenterX + \radius * cos(\angle*6)}, {\radius * sin(\angle*6)}) {};

    \foreach \i in {1, 2, 3, 4, 5, 6} {
        \node[roundnode] (a\i) at ({5 + \ACenterX + \radius * cos(\angle*\i)}, {\radius * sin(\angle*\i) + 1.2}) {\i};
    }

    \draw (a1) -- (a6);
    \draw (a1) -- (a5);
    \draw (a2) -- (a3);
    \draw (a2) -- (a4);
    \draw (a2) -- (a5);
    \draw (a2) -- (a6);
    \draw (a3) -- (a5);
    \draw (a5) -- (a6);

    \node (av) at ({4.5 + \ACenterX + \radius * cos(\angle*3)}, {\radius * sin(\angle*3) + 1.2}) {};


    \foreach \i in {1, 2, 3, 4, 5, 6} {
        \node[roundnode] (b\i) at ({5 + \ACenterX + \radius * cos(\angle*\i)}, {\radius * sin(\angle*\i) - 1.2}) {\i};
    }

    \draw (b1) -- (b6);
    \draw (b1) -- (b5);
    \draw (b2) -- (b4);
    \draw (b2) -- (b5);
    \draw (b2) -- (b6);
    \draw (b3) -- (b5);
    \draw (b4) -- (b6);
    \draw (b5) -- (b6);

    \node (bv) at ({4.5 + \ACenterX + \radius * cos(\angle*3)}, {\radius * sin(\angle*3) - 1.2}) {};

    \node at ({3 + \ACenterX + \radius * cos(\angle*3)}, {\radius * sin(\angle*3) + 1.8}) {$\texttt{AddEdge}(2, 3)$};
    \node at ({3 + \ACenterX + \radius * cos(\angle*3)}, {\radius * sin(\angle*3) - 1.8}) {$\texttt{AddEdge}(4, 6)$};

    \draw[->, thick] (nv) to[out=20, in=180, looseness=0.8] (av);
    \draw[->, thick] (nv) to[out=-20, in=180, looseness=0.8] (bv);
    
\end{tikzpicture}
    \caption{Two actions induce isomorphic graphs, making them transition-equivalent. However, they are not orbit-equivalent. This example was originally presented by \citep{ma2024baking}.}
    \label{fig:counter example}
    \vspace{-30pt}
\end{wrapfigure}

In \thmref{thm:equivalence implication}, we established that orbit equivalence implies transition equivalence. However, the converse is not generally true, though such cases are rare. A counterexample is illustrated in \figref{fig:counter example}. While the two resulting graphs are isomorphic (by the permutation 1→4, 2→5, 3→1, 4→3, 5→6, 6→2), they are induced by actions that modify nodes belonging to different orbits. In other words, two actions $\texttt{AddEdge}(2, 3)$ and $\texttt{AddEdge}(4, 6)$ are transition-equivalent, but not orbit-equivalent. However, accounting for orbit-equivalent actions is sufficient, as the corresponding backward actions belong to two distinct orbits as well.

\section{Properties of Graph Neural Networks}
\label{app:GNN-property}

\subsection{Permutation Equivariance}
\label{app:permutation equivariance}

The key design principle of GNNs is permutation equivariance, which ensures that the output remains consistent regardless of how the nodes in the input graph are ordered. 

\begin{definition}[Permutation Equivariance]
    A function $f$ is permutation-equivariant if it satisfies $f(\pi(x)) = \pi(f(x))$ for any permutation $\pi$. 
\end{definition}

Specifically, let $\mathbf{A} \in \R^{n\times n \times d}$ be the adjacency tensor of a graph $G$ with $n$ nodes. The $d$ dimensional node and edge features of $G$ are represented in $\mathbf{A}$, where diagonal elements encode the node features. Let $\mathbf{A}[i, j]$ represent the $(i, j)$-th element of the tensor. We define the permutation of the tensor as $\pi(\mathbf{A})[i, j] = \mathbf{A}[\pi^{-1}(i), \pi^{-1}(j)]$. Automorphisms are permutations that preserve adjacency tensor. That is, for $\pi \in \Aut(G)$, we have $\pi(\mathbf{A})=\mathbf{A}$. Since GNNs are permutation-equivariant, we can show that they produce identical node representations for nodes in the same orbit.

\begin{theorem}
    Let $f: \R^{n\times n \times d} \times \R^{n\times n \times d}$ be a permutation-equivariant function. Then, for any $u, v, h, k \in V$, if there exists a permutation $\pi \in \Aut(G)$ such that $\pi(u) = h$ and $\pi(v) = k$, it follows that $f(\mathbf{A})[u, v] = f(\mathbf{A})[h, k]$.
\end{theorem}

\begin{proof}
    \begin{align*}
    f(\mathbf{A})[u, v] 
        &= \pi^{-1}(f(\pi (\mathbf{A})))[u, v] \\
        &= f(\pi(\mathbf{A}))[\pi(u), \pi(v)] \\
        &= f(\mathbf{A})[h, k].
    \end{align*}
\end{proof}

The theorem implies that nodes and edges within the same orbit are represented identically by GNNs.

\subsection{Expressive Power}
\label{app:Expressive Power of GNNs}

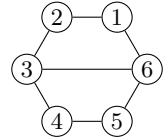
\begin{wrapfigure}{r}{0.3\textwidth}
    \vspace{-2pt}
    \centering
    \begin{tikzpicture}[
    roundnode/.style={circle, draw, minimum size=4mm, inner sep=0pt, font=\small}
]
    \def\radius{0.8}

    \def\angle{360/6}
    
    \foreach \i in {1, 2, 3, 4, 5, 6} {
        \coordinate (n\i) at ({1.2 + \radius * cos(\angle*\i)}, {\radius * sin(\angle*\i)});
    }

    \foreach \i in {1, 2, 3, 4, 5, 6} {
        \node[roundnode] (n\i) at (n\i) {$\i$};
    }

    \draw (n1) -- (n2);
    \draw (n2) -- (n3);
    \draw (n3) -- (n4);
    \draw (n4) -- (n5);
    \draw (n5) -- (n6);
    \draw (n6) -- (n1);
    \draw (n3) -- (n6);

\end{tikzpicture}
    \caption{Two edges $(2, 4)$ and $(2, 5)$ belong to different orbits and will result in non-isomorphic graphs if added to the graph. However, their edge representations will be identical if aggregated from node representations.}
    \label{fig:GNN-failure}
    \vspace{-4pt}
\end{wrapfigure}

Another source of bias comes from the expressiveness of GNNs used to parameterize the policy. If actions from different orbits collapse into identical representations, reducing the network's representational power, the policy may be forced to assign equal probabilities to actions that produce different rewards. In \figref{fig:GNN-failure}, the two actions $\texttt{AddEdge}(G, 2, 4)$ and $\texttt{AddEdge}(G, 2, 5)$ will have identical representations if they are aggregated from node representations. This occurs because nodes $2, 4$ and $5$ all belong to the same orbit and, therefore, share identical representations.

\begin{wrapfigure}{tr}{0.45\textwidth}
    \vspace{-30pt}
    \centering
    \includegraphics[width=0.45\textwidth]{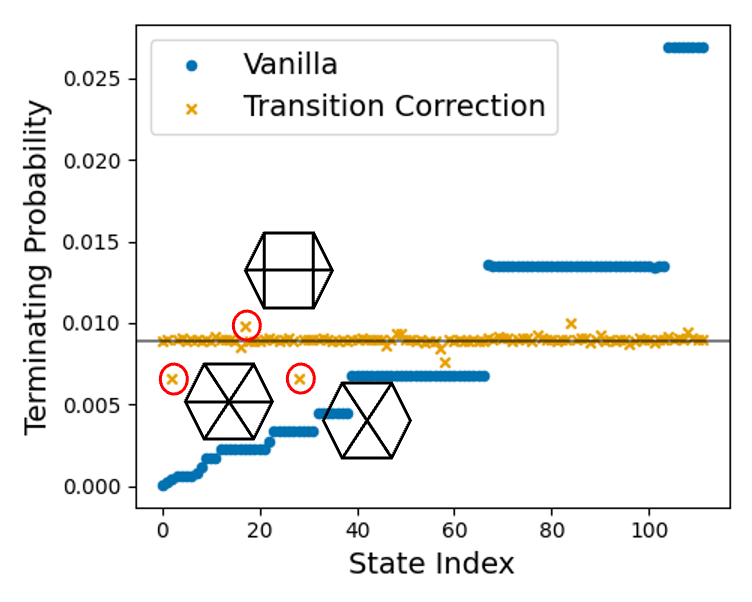}
    \caption{Terminating probabilities. States are sorted according to the number of graphs in the state, $|x|$. States with large errors are marked with red circles.}
    \label{fig:uniform-GNN-failure}
    \vspace{-30pt}
\end{wrapfigure}

Significant research has been conducted on techniques to enhance the expressive power of GNNs \citep{xu2018powerful, dwivedi2021graph}. While much of this work has focused on improving graph-level representations, some methods have been proposed to enhance multi-node representations, such as edges \citep{zhang2021labeling}. In graph generation tasks, actions are often parameterized using all levels of representations---node-level, edge-level, and graph-level. This complexity makes designing expressive GNNs more challenging, emphasizing the need for architectures with enhanced expressive power when the task requires it.

In our synthetic experiments in the main text, we augmented edge-level representations from the GNN with shortest path lengths to distinguish edges in different orbits. For example, in \figref{fig:GNN-failure}, shortest path length of the edge $(2, 4)$ is 2, while the length of $(2, 5)$ is 3. We conducted experiments without this feature augmentation, resulting in an inexact policy, as shown in \figref{fig:uniform-GNN-failure}. Compare this to \figref{fig:synthetic results} (a), where the terminating probabilities are more accurate. In \figref{fig:uniform-GNN-failure}, terminating probabilities for certain states exhibit significantly larger errors, with those states marked by red circles. These states correspond to successor states derived from the state shown in \figref{fig:GNN-failure}.

\section{Proofs}
\label{app:proofs}

\subsection{Proof of Equation \ref{eqn:state transition probability}}
\label{app:state transition probability}

While we did not present \eqnref{eqn:state transition probability} as a theorem, a formal derivation offers insights into the relationship between $(\gS, \gA)$ and $(\gG, \gE)$. Specifically, \eqnref{eqn:state transition probability} states that

\begin{equation}
\label{eqn:state transition probability}
    p_\gA(s'|s) = \sum_{G' \in \gE(G) \cap s'} p_\gE(G'|G).
\end{equation}

for any $G \in s$.

\begin{proof}
We expand state transition probability $p_\gA(s'|s)$ in terms of graph transitions as follows:

\begin{align*}
    p_\gA(s'|s) 
        &= \frac{
            p_\gA(s'|s)\bar p_\gA(s) 
        }{
            \bar p_\gA(s)
        } \\
        &= \frac{
            \sum_{G\in s}\sum_{G'\in\gE(G)\cap s'}
                p_\gE(G'|G)\bar p_\gE(G) 
        }{
            \sum_{G \in s} \bar p_\gE(G)
        } \\
        &= \frac{
            \sum_{G\in s}\bar p_\gE(G)
            \sum_{G'\in\gE(G)\cap s'} p_\gE(G'|G)
        }{
            \sum_{G \in s} \bar p_\gE(G)
        } 
\end{align*}

If \( \sum_{G'\in\gE(G)\cap s'} p_\gE(G'|G) \) is constant for all \( G \in s \), we can factor it out, obtaining:
\[
p_\gA(s'|s) 
= \sum_{G' \in \gE(G_0) \cap s'} p_\gE(G'|G_0), \quad \text{for any } G_0 \in s.
\]

This constantness follows from two assumptions: 1) $p_\gE$ is permutation equivalent, and 2) $\gE$ is structured, meaning for any isomorphic graphs $G, G' \in s$, the set of next graphs are matched by some permutation such that \( \gE(G) = \pi(\gE(G')) \). In other words,
\[
\sum_{G'\in\gE(G)\cap s'} p_\gE(G'|G) 
= \sum_{G'\in\gE(\pi(G))\cap s'}p_\gE(G'|\pi(G))
\]
holds for all \( \pi \in \Aut(G) \).
\end{proof}

\subsection{Proof of Theorem \ref{thm:equivalence implication}}
\label{app:proof of thm:equivalence implication}

The theorem simplifies to the assertion that, for a given graph $G$, graph actions of the same type applied within the same orbit are transition-equivalent. To establish this, we prove the theorem for each action type. The results are straightforward for attribute-level actions, and we provide a proof for the \texttt{SetNodeAttribute} action.

\begin{lemma}[SetNodeAttribute]
    Let $G[l_n(u) = t]$ denote the graph where the attribute of node $u$ in graph $G$ is changed to $t$. If $\Orb(G, u) = \Orb(G, v)$, then $G[l_n(u) = t] \cong G[l_n(v) = t]$.
\end{lemma}
\begin{proof}
    Let us denote the node labeling function of $G$, $G[l_n(u) = t]$ and $G[l_n(v) = t]$ as $l_n$, $l_{n[u=t]}$ and $l_{n[v=t]}$ respectively. Since $u$ and $v$ are in the same orbit, there exists $\pi \in \Aut(G)$ such that $\pi(u) = v$. Showing $l_{n[u=t]}(w) = l_{n[v=t]}(\pi(w))$ for all $w \in V$ is sufficient to establish the isomorphism. 
    We prove for two cases. First, let $w = u$. Then, $\pi$ satisfies $l_{n[u=t]}(w) = l_{n[u=t]}(u) = l_{n[v=t]}(v) = l_{n[v=t]}(\pi(u)) = l_{n[v=t]}(\pi(w))$. Secondly, let $w \neq u$. Then, $l_{n[u=t]}(w) = l_n(w) = l_n(\pi(w)) =  l_{n[v=t]}(\pi(w))$.
\end{proof}

The proof for the \texttt{SetEdgeAttribute} action is nearly identical to that for \texttt{SetNodeAttribute}, with nodes replaced by edges. We now proceed to prove the cases for the \texttt{AddEdge}, \texttt{AddNode}, and their corresponding backward actions. We use the following two properties in our proofs.

\begin{align*}
    \pi(E\cup E') 
    &= \{(\pi(u), \pi(v)):(u, v) \in E\cup E'\}   \\
    &= \{(\pi(u), \pi(v)):(u, v) \in E \text{ or } (u, v)\in E'\}   \\
    &= \{(\pi(u), \pi(v)):(u, v) \in E\} \cup\{(\pi(u), \pi(v)): (u, v)\in E'\}   \\
    &= \pi(E) \cup \pi(E'),
\end{align*}

and similarly,

\begin{align*}
    \pi(E\setminus E') 
    &= \{(\pi(u), \pi(v)):(u, v) \in E\setminus E'\}   \\
    &= \{(\pi(u), \pi(v)):(u, v) \in E \text{ and } (u, v) \notin E'\}   \\
    &= \{(\pi(u), \pi(v)):(u, v) \in E\} \setminus \{(\pi(u), \pi(v)): (u, v)\in E'\}   \\
    &= \pi(E) \setminus \pi(E'), \\
\end{align*}

where $E$ and $E'$ are edge sets. We assume homogeneous graphs for simplicity.

\begin{lemma}[AddEdge]
    \label{appendix:lemma:AddEdge}
    Let $G[E \cup (u, v)]$ and $G[E \cup (h, k)]$ denote the graphs induced by $E \cup \{(u, v)\}$ and $E \cup \{(h, k)\}$, respectively. If $(u, v)$ and $(h, k)$ are in the same orbit in $G$, then $G[E \cup (u, v)]$ and $G[E \cup (h, k)]$ are isomorphic. In other words, $\Orb(G, u, v) = \Orb(G, h, k)$ implies $G[E \cup (u, v)] \cong G[E \cup (h, k)]$.
\end{lemma}
\begin{proof}
    If $(u, v)$ and $(h, k)$ are in the same orbit, then there exists $\pi \in \Aut(G)$ such that $(\pi(u), \pi(v)) = (h, k)$. Since $\pi$ is an automorphism, it also satisfies $\pi(E) = E$. Thus, $\pi(E \cup \{(u, v)\}) = \pi(E) \cup \{(\pi(u), \pi(v))\} = E \cup \{(h, k)\}$, indicating that $\pi$ is an isomorphism between $G[E \cup (u, v)]$ and $G[E \cup (h, k)]$.
\end{proof}

\begin{lemma}[RemoveEdge]
    \label{appendix:corollary:RemoveEdge}
    Let $G[E\setminus(u, v)]$ and $G[E\setminus(h, k)]$ denote graphs induced by $E\setminus\{(u, v)\}$ and $E\setminus\{(h, k)\}$ respectively. Then, $G[E\setminus(u, v)]$ and $G[E\setminus(h, k)]$ are isomorphic if $(u, v)$ and $(h, k)$ are in the same orbit in graph $G$.
\end{lemma}
\begin{proof}
    Let $\pi \in \Orb(G, u, v)$ such that $(\pi(u), \pi(v))=(h, k)$. Then, $\pi(E\setminus\{(u,v)\}) = \pi(E)\setminus\{(\pi(u),\pi(v))\}) = E\setminus\{(h, k)\}$, completing the proof.
\end{proof}

\begin{lemma}[AddNode]
    \label{appendix:lemma:AddNode}
    Let $G[E \cup (u, w)]$ and $G[E \cup (v, w)]$ denote the graphs induced by attaching a new node $w$ to the existing nodes $u$ and $v$, respectively. Then, $\Orb(G, u) = \Orb(G, v)$ implies $G[E \cup (u, w)] \cong G[E \cup (v, w)]$.
\end{lemma}
\begin{proof}
    Let $\pi \in \Orb(G, u)$ such that $\pi(u)=v$, and let $\Tilde{\pi}: V\cup\{w\} \rightarrow V\cup\{w\}$ be the extension of  $\pi$ such that $\Tilde{\pi}(w) = w$ and $\Tilde{\pi}(i)=\pi(i)$ for $i \neq w$. Then, $\Tilde{\pi}(E\cup (u, w)) = \Tilde{\pi}(E)\cup (\Tilde{\pi}(u), \Tilde{\pi}(w)) = \pi(E)\cup (\pi(u), \Tilde{\pi}(w))  = E\cup(v, w)$.
\end{proof}

\begin{corollary}[RemoveNode]
\label{appendix:corollary:RemoveNode}
    Let $G[V \setminus u]$ and $G[V \setminus v]$ denote the graphs induced by removing nodes $u$ and $v$, respectively, where edges connected to $u$ or $v$ are also removed. Then, $\Orb(G, u) = \Orb(G, v)$ implies $G[V \setminus u] \cong G[V \setminus v]$.
\end{corollary}
\begin{proof}
    Let $\pi \in \Aut(G)$ be an automorphism of $G=(V, E)$ such that $\pi(u)=v$. Then:
    \begin{align*}
        \pi\big (E\setminus(\{(u, k): k\in V\}\cup\{(k, u): k\in V\})\big ) 
        &= \pi(E)\setminus(\{(\pi(u), \pi(k)): k\in V\}\cup\{(\pi(k), \pi(u)): k\in V\})\\
        &= E\setminus(\{(v, k'): k'\in V\}\cup\{(k', v): k'\in V\})
    \end{align*}
    where $E\setminus(\{(u, k): k\in V\}\cup\{(k, u): k\in V\})$ and $E\setminus(\{(v, k'): k'\in V\}\cup\{(k', v): k'\in V\})$ are edge sets of the induced subgraphs $G[V \setminus u]$ and $G[V \setminus v]$, respectively. Therefore, $\pi$, restricted to $V \setminus \{u\}$, is an isomorphism between $G[V \setminus u]$ and $G[V \setminus v]$. 
\end{proof}

\subsection{Proof of Theorem \ref{thm:orbit equivalence sufficiency}}

\begin{proof}
Note that orbit equivalence implies transition equivalence by \thmref{thm:equivalence implication}, meaning that each set of transition equivalent actions can be partitioned into subsets of orbit-equivalent actions. State transition probabilities can thus be computed by summing over state-action probabillities $p_{\bar\gA}(a|s)$, where each $p_{\bar\gA}$ is in turn defined by summing over orbit-equivalent graph actions in $a$.

Formally, let $\bar\gA(s)$ be the set of actions available from state $s$, and let $a(s)$ represent the next state obtained by applying action $a$ to $s$. Define $\bar\gA(s, s')=\{a \in \bar\gA(s):a(s)=s'\}$ as the set of forward actions from state $s$ that lead to state $s'$. Then, state transition probability can be expressed as:

\begin{align*}
    p_{\gA}(s'|s) 
        &= \sum_{a \in \bar\gA(s, s')} p_{\bar\gA}(a|s) \\
    q_{\gA}(s|s')
        &= \sum_{a \in \bar\gA(s, s')} q_{\bar\gA}(a|s').
\end{align*}

From the assumption, state-action flow constraints are satisfied, i.e., 

\begin{align*}
    F(s)p_{\bar\gA}(a|s) = F(s')q_{\bar\gA}(a|s').
\end{align*}

The assumption can be satisfied if the equivariant neural networks parameterizing $p_{\bar\gA}$ and $q_{\bar\gA}$ can distinguish between orbits differently, and the number of orbits (and hence the number of actions) is the same for both the forward and backward transitions. 

Summing both side of the equations over $\bar\gA(s, s')$, we have state transition flow constraints, $F(s)p_{\gA}(s'|s) = F(s')q_{\gA}(s|s')$.

\end{proof}

\subsection{Proof of Lemma \ref{lem:orb-aut} and Its Generalization}
\label{app:proof and generalization of orb-aut}

\lemref{lem:orb-aut} relates the order of orbits to the order of automorphism group. The proof relies on orbit-stabilizer theorem, hence we introduce the following definition.

\begin{definition}[Stabilizer]
    The stabilizer of a node $u \in V$ in graph $G$ is the set of automorphisms that fix node $u$: $\Stab(G, u) =  \{\pi \in \Aut(G): \pi(u) = u\}$. The stabilizer of an edge $(u, v)$ is defined as $\Stab(G, u, v) =  \{\pi \in \Aut(G): \pi(u) = u, \pi(v) = v\}$. Similarly, the stabilizer of a node set $S$ is defined as $\Stab(G, S) =  \{\pi \in \Aut(G): \pi(S)=S\}$.
\end{definition}

We restate \lemref{lem:orb-aut} and provide its proof.

\begin{lemma}[AddEdge]
\label{lem:orb-aut:AddEdge}
    Let $G = G'[E' \setminus (u, v)]$ and $G' = G[E \cup (u, v)]$ be two successive graphs induced be $E' \setminus (u, v)$ and $E \cup (u, v)$. Then the following equation holds:
    $$
    \frac{|\Orb(G, u, v)|}{|\Orb(G', u, v)|}
    = \frac{|\Aut(G)|}{|\Aut(G')|}.
    $$
\end{lemma}

\begin{proof}
    Using the orbit-stabilizer theorem, we have
    \begin{equation*}
    \begin{aligned}
    \frac{|\Orb(G, u, v)|}{|\Orb(G', u, v)|} = \frac{|\Aut(G)|}{|\Aut(G')|}
    &\iff
    \frac{|\Aut(G)|}{|\Orb(G, u, v)|} = \frac{|\Aut(G')|}{|\Orb(G', u, v)|}
    \\
    &\iff
    |\Stab(G, u, v)| = |\Stab(G', u, v)|.
    \end{aligned}
    \end{equation*}

    Hence, we prove the lemma by showing $|\Stab(G, u, v)| = |\Stab(G', u, v)|$. It suffices to prove that $\Stab(G, u, v) = \Stab(G', u, v)$. 
    
    First, we show that $\Stab(G, u, v) \subseteq \Stab(G', u, v)$. Let $\pi \in \Stab(G, u, v)$. Then $\pi(u)=u$, $\pi(v)=v$, and $\pi(E \cup \{(u, v)\}) = \pi(E) \cup \{(\pi(u), \pi(v))\} = E \cup \{(u, v)\}$, which implies that $\pi \in \Stab(G', u, v)$.

    Conversely, let $\pi' \in \Stab(G', u, v)$, so $\pi'(u)=u$, $\pi'(v)=v$, and $\pi'(E\cup\{(u, v)\}) = \pi'(E) \cup \{(u, v)\} = E \cup(u, v)$. Suppose for contradiction that $\pi'(E) \neq E$. Then some edge in $E$ must be mapped to a different edge not in $E$, and to satisfy the equality $\pi'(E) \cup \{(u, v)\} = E \cup(u, v)$, the only possibility is that $\pi'(E)\setminus E=\{(u, v)\}$. But since $\pi'(\{(u,v)\})=\{(u,v)\}$, this implies a duplication, contradicting the assumption that $\pi'$ is a permutation. Therefore, $\pi'(E)=E$, and hence $\pi'\in\Stab(G, u, v)$.

    Thus, $\Stab(G, u, v) = \Stab(G', u, v)$, completing the proof.
\end{proof}

For general statement of \lemref{lem:orb-aut}, we define the orbit of a graph action as the orbit of set of nodes or edges affected by the action. For example, the orbit of $e= \texttt{AddEdge}(G, u, v)$, denoted as $\Orb(G, e)$, corresponds to $\Orb(G, u, v)$. The backward action associated with $e$ is $\texttt{RemoveEdge}(G', u, v)$, and the orbit of this action when applied to the next graph $G'$ is denoted as $\Orb(G', e)$, which corresponds to $\Orb(G', u, v)$. Definition of orbits for each action type is provided in \tabref{tab:action-orbit}. 

\begin{table}[h]
\caption{Orbit of graph actions}
\label{tab:action-orbit}
\begin{center}
\begin{tabular}{ll}
\hline
    \texttt{AddNode}$(G, u)$ & $\Orb(G, u)$ \\
    \texttt{RemoveNode}$(G, v)$ & $\Orb(G, v)$ \\
    \texttt{AddEdge}$(G, u, v)$ & $\Orb(G, u, v)$ \\
    \texttt{RemoveEdge}$(G, u, v)$ & $\Orb(G, u, v)$ \\
    \texttt{SetNodeAttribute}$(G, u, t)$ & $\Orb(G, u)$ \\
    \texttt{SetEdgeAttribute}$(G, u, v, t)$ & $\Orb(G, u, v)$ \\
    \texttt{SetGraphAttribute}$(G, t)$ & $\Orb(G, V)$ \\
\hline
\end{tabular}
\end{center}
\end{table}

Using these definitions, we generalize \lemref{lem:orb-aut} as follows.

\begin{lemma}
    \label{lem:generalized-orb-aut}
    Let $(G, G') \in \gE$ be a graph transition induced by an action $e \in \bar\gE$ that either adds a node, an edge, or modifies an attribute in the graph $G$. Then, the following relationship holds:
    \begin{equation*}
    \frac{\text{Number of forward orbit-equivalent actions}}{\text{Number of backward orbit-equivalent actions}}
    =\frac{|\Orb(G, e)|}{|\Orb(G', e)|}
    = \frac{|\Aut(G)|}{|\Aut(G')|}.
    \end{equation*}
\end{lemma}

We already proved \lemref{lem:generalized-orb-aut} for \texttt{AddEdge}.

For \texttt{AddNodeAttribute}$(G, u, t)$, the proof is straightforward: the only difference between $G$ and $G'$ is the attribute assigned to node $u$, so it is immediate that $\Stab(G, u) = \Stab(G', u)$. The argument for \texttt{AddEdgeAttribute} is nearly identical. Next, we prove the corresponding result for the \texttt{AddNode}.

\begin{lemma}[AddNode]
Let $G'=(V', E')$ be the graph resulted by adding node $v$ to node $u$ in graph $G=(V, E)$. Then, 

$$
\frac{|\Orb(G, u)|}{|\Orb(G', v)|}
= \frac{|\Aut(G)|}{|\Aut(G')|}.
$$ 
\end{lemma}

\begin{proof}
Using the orbit-stabilizer theorem, it suffices to show $|\Stab(G, u)| = |\Stab(G', v)|$. We do this by constructing a bijective map $f: \Stab(G, u) \rightarrow \Stab(G', v)$. Define $f(\pi)=\tilde\pi$ for $\pi\in \Stab(G, u)$, where $\tilde\pi$ is the extension of $\pi$ defined on $V\cup\{v\}$ given by 

\[ \tilde\pi(w) =
  \begin{cases}
    \pi(w)       & \quad \text{if } w \in V\\
    v  & \quad \text{if } w=v.
  \end{cases}
\]

We claim that $\tilde\pi \in \Stab(G', v)$. Since $\pi \in \Aut(G)$ and $\pi(u)=u$, it follows that $\tilde\pi\in\Aut(G')$ and $\tilde\pi(v)=v$. Thus, $f$ maps into $\Stab(G', v)$ and is injective by construction.

To show that $f$ is surjective, let $\pi'\in \Stab(G', v)$. Then $\pi'\in\Aut(G')$ and $\pi'(v)=v$, and we have

\[
    \pi'(E\cup\{(u, v)\})=\pi'(E)\cup\{(\pi'(u), v)\}=E\cup\{(u, v)\}.
\]

This implies that $\pi'(E)=E$ and $(\pi'(u), v)=(u, v)$, so $\pi'(u)=u$, since $\pi'(E)$ cannot contain $(u, v)$. Therefore, the restriction of $\pi'$ on $V$, denoted $\pi$, belongs to $\Stab(G, u)$, and satisfies $f(\pi)=\pi'$.

Hence, $f$ is bijective, and $|\Stab(G, u)| = |\Stab(G', v)|$, as claimed.
\end{proof}

\subsection{Proof of Theorem \ref{thm:Automorphism correction}}

\begin{proof}
We first note that state-action probability can be computed by multiplying the number of orbit-equivalent actions when $p_\gE$ is permutation-equivariant:

\begin{align*}
    p_{\bar\gA}(a|s)
        &= \sum_{e \in \bar\gE(G) \cap a} p_{\bar\gE}(e|G) \\
        &= |\bar\gE(G) \cap a| \cdot p_{\bar\gE}(e|G) 
\end{align*}

for any $G \in s$. Since $|\bar\gE(G) \cap a|$ represents the number of orbit-equivalent actions from graph $G$, it is equal to $|\Orb(G, e)|$ for any $e \in \bar\gE(G) \cap a$, where we defined the orbit of graph actions in \appref{app:proof and generalization of orb-aut}. Thus,

\begin{align*}
    p_{\bar\gA}(a|s) = |\Orb(G, e)| \cdot p_{\bar\gE}(e|G).
\end{align*}

Similarly, for the backward policy, we have 

\begin{align*}
    q_{\bar\gA}(a|s') = |\Orb(G', e)| \cdot q_{\bar\gE}(e|G').
\end{align*}

We prove \thmref{thm:Automorphism correction} using the following sequence of equations.

\begin{align*}
\frac{p_{\bar\gA}(a|s)}{q_{\bar\gA}(a|s')} 
    &= \frac{
        |\Orb(G, e)| \cdot p_{\bar\gE}(e|G)
    }{
        |\Orb(G', e)| \cdot q_{\bar\gE}(e|G')
    }\\
    &= \frac{|\Aut(G)|}{|\Aut(G')|}
    \cdot
    \frac{p_\gE(e|G)}{q_\gE(e|G')},
\end{align*}

where we used \lemref{lem:generalized-orb-aut} for the last equation.
\end{proof}

\subsection{Proof of Theorem \ref{thm:DB-correction}}
\label{app:proof:DB-fixed}

Before proving \thmref{thm:DB-correction}, we first prove the existence of a policy that satisfies graph-level DB constraints.

\begin{lemma}
    For any given reward function $R$, there exist $p_\gE$, $q_\gE$, and $\Tilde{F}$ that satisfy the graph-level detailed balance constraints for all transitions $ (G, G') \in \gE$, defined as follows:
    \begin{align}
        \Tilde{F}(G)p_\gE(G'|G) = \Tilde{F}(G')q_\gE(G|G')
    \end{align}
    Note that this differs from the usual state-level detailed balance condition:
    \begin{align}
        F(s)p_{\bar\gA}(a|s) = F(s')q_{\bar\gA}(a|s').
    \end{align}
\end{lemma}

\begin{proof}
    By \thmref{thm:Automorphism correction}, state-level detailed balance constraints can be rewritten as graph transition probabilities as follows:
    \begin{align}
    |\Aut(G)|F(G)p_{\gE}(G'|G)=|\Aut(G')|F(G')q_{\gE}(G|G').
    \end{align}
    Defining $\Tilde{F}(G) = |\Aut(G)|F(G)$, the functions $\Tilde{F}$, $p_{\gE}$, and $q_{\gE}$ satisfy the graph-level detailed balance constraints for a given $R$.
\end{proof}

\begin{figure}[t]
\centering
    \includegraphics[width=0.9\textwidth]{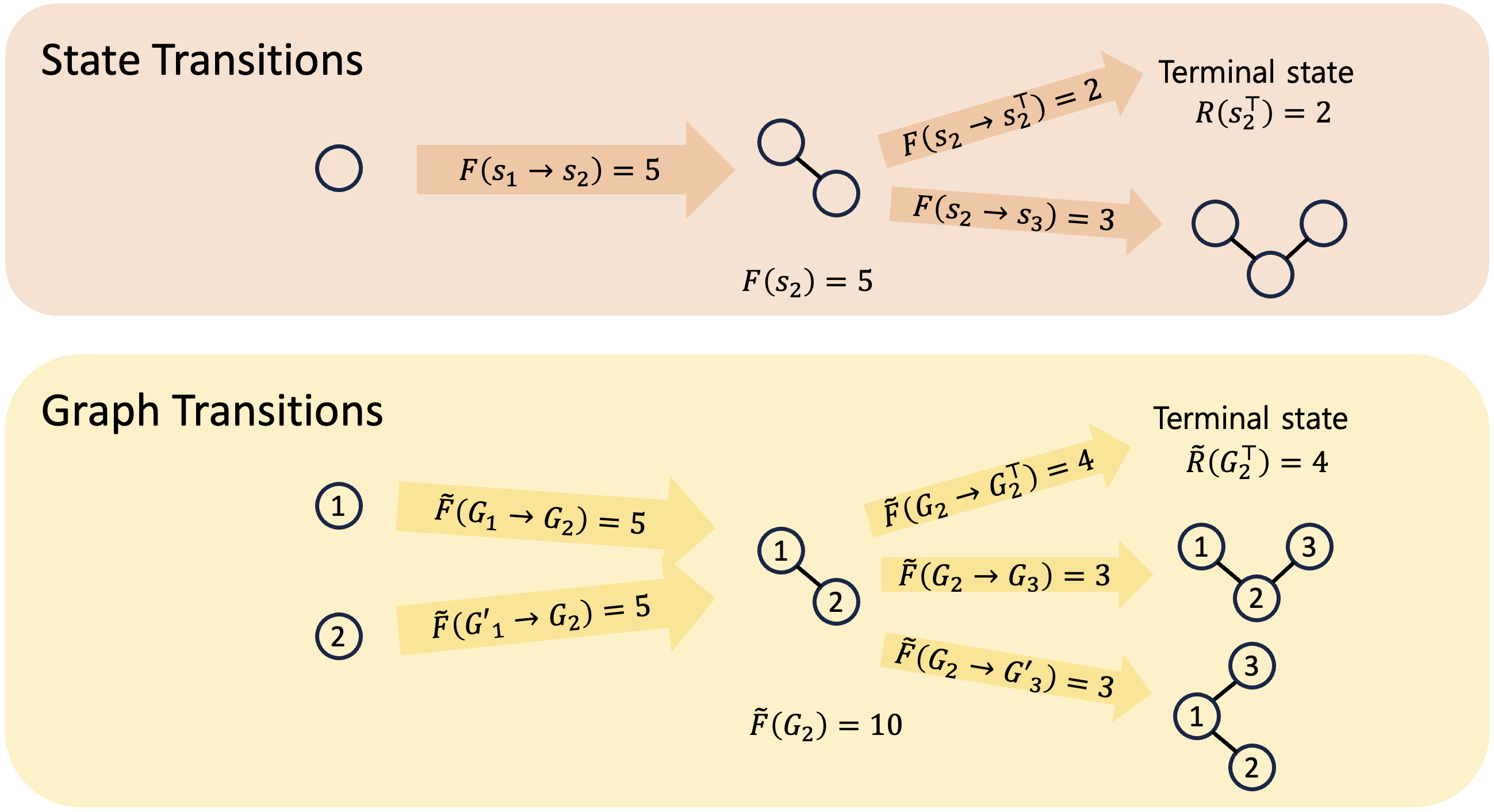}
    \caption{
    Illustration of the effect of reward adjustment. \textbf{Above}: State transitions from and to $s_2$. \textbf{Below}: Graph transitions from and to $G_2$. Due to the effect of the scaled reward $\Tilde{R}$, state flows are also scaled by $|\Aut(G)|$, leading to $\Tilde{F}(G) = F(s) |\Aut(G)|$. The edge flows remain unchanged in this figure. Note that the graph-level detailed balance condition holds, while the termination probability is proportional to $R(s)$.
    }
    \label{fig:graph-space}
\end{figure}

\begin{theorem}[Restatement of \thmref{thm:DB-correction}]
    If the rewards are scaled by $|\Aut(G)|$ and the graph-level detailed balance constraints are satisfied for $p_\gE$, $q_\gE$, and $\Tilde{F}$, then the corresponding forward policy will sample proportionally to the reward.
\end{theorem}
\begin{proof}
    For a given complete trajectory $G_0, \dots, G_n$, we have:
    \begin{align*}
        \Tilde{F}(G_0)p_\gE(G_1|G_0) &= \Tilde{F}(G_1)q_\gE(G_0|G_1), \\
        &\cdots \\
        \Tilde{F}(G_{n-1})p_\gE(G_n|G_{n-1}) &= |\Aut(G_n)|R(G_n)q_\gE(G_{n-1}|G_n).
    \end{align*}
    Multiplying the left- and right-hand sides of all the equations, we get:
    \[
    \Tilde{F}(G_0)\prod_{t=0}^{n-1}p_\gE(G_{t+1}|G_t)
    = |\Aut(G_n)|R(G_n)\prod_{t=0}^{n-1}q_\gE(G_t|G_{t+1}).
    \]
    Defining $\Tilde{F}(G_0) = Z$, this reduces to the state-level trajectory balance condition with corrections as in \corref{cor:TB-correction}, which ensures $\bar p_\gA(x) \propto R(x)$, as shown by Proposition 1 of \citet{malkin2022trajectory}.
\end{proof}

\section{Discussion on the Flow-Matching Objective}
\label{app:Discussion on the Flow-Matching Objective}
The first GFlowNet training objective proposed by \citet{bengio2021flow} is the Flow-Matching (FM) objective, which requires the flow-matching condition to hold at every state $s'$:

\begin{align}
    \sum_{s:(s, s')\in\gA} F(s \rightarrow s') = 
    \sum_{s'':(s', s'')\in\gA} F(s' \rightarrow s'')
\end{align}

where $F(s \rightarrow s')$ denotes the edge flow. As with the DB and TB objectives, it suffices to scale the final rewards to remove the bias introduced by action equivalence.

\begin{corollary}[FM correction]
    Consider the node-by-node generation. Specifically, we restrict the set of graph actions to those defined in \appref{app:definitions of actions}, except for fragment-based actions. Define the graph-level flow-matching condition as:
    \[
        \sum_{G:(G, G')\in\gE} \tilde F(G\rightarrow G') 
        = \sum_{G'':(G',G'')\in\gE} \tilde F(G'\rightarrow G''),
    \]
    where $\tilde F(\cdot \rightarrow \cdot)$ denotes a permutation equivalent graph-level edge-flow function. If the rewards are given by $\tilde R(G) = |\Aut(G)|R(G)$ and the graph-level flow-matching condition holds for all states, then the forward policy induced by $\tilde F$ samples terminal states proportionally to the original reward $R$.
\end{corollary}
\begin{proof}
    The detailed balance condition is satisfied whenever the flow-matching condition holds, via the relations:
    \begin{equation*}
    \begin{aligned}
    \tilde F(G) &= \sum_{G'\in\gE(G)}\tilde{F}(G\rightarrow G'), \\
    p_\gE(G'|G) &= \frac{\tilde{F}(G\rightarrow G')}{\tilde F(G)}, \\
    q_\gE(G|G') &= \frac{\tilde{F}(G \rightarrow G')}{\tilde F(G')}.
    \end{aligned}
    \end{equation*}
    Applying \thmref{thm:DB-correction}, the result follows.
\end{proof}

As with the DB objective, FM can be corrected at each state without scaling the reward. Consider a neural network that parameterizes the edge flow function and outputs all possible outgoing edge flows $\tilde{F}(G \rightarrow G')$ from a given graph $G$. Outflows, defined as $\sum_{G'':(G',G'')\in\gE} \tilde F(G'\rightarrow G'')$, can be computed easily by summing the outputs of the neural network. In contrast, computing the inflows for a given graph $G'$ involves two considerations: (1) enumerating parents of $G'$, which may include isomorphic duplicates. To account for this, the total inflow should be divided by the number of duplicates; (2) computing edge flows $\tilde F(G \rightarrow G')$ for each parent graph $G$, which may involve equivalence actions. To account for this, the total inflow should be multiplied by the number of equivalent actions. These considerations lead to the following result.

\begin{theorem}
    Let $G'$ be the given graph representing the current state. Assuming node-by-node generation, the state-level flow-matching condition can be expressed in terms of the graph-level flow-matching condition as follows:
    \[
        \sum_{G:(G, G')\in\gE} |\Aut(G)|\tilde F(G\rightarrow G') = |\Aut(G')|\sum_{G'':(G',G'')\in\gE(G)}  \tilde F(G'\rightarrow G''),
    \]
    where $\tilde F(\cdot \rightarrow \cdot)$ denotes a permutation-equivalent edge-flow function.
\end{theorem}
\begin{proof}
    The inflows are computed as follows:
    \begin{align*}
    \sum_{s:(s, s')\in\gA} F(s\rightarrow s')
    &=\sum_{G:(G, G')\in\gE} 
    \frac{\text{Number of forward orbit-equivalent actions from $G$ to $G'$}}{\text{Number of backward orbit-equivalent actions from $G'$ to $G$}}\tilde F(G\rightarrow G') \\
    &=\sum_{G:(G, G')\in\gE} \frac{|\Aut(G)|}{|\Aut(G')|}\tilde F(G\rightarrow G'),
    \end{align*}
    where the last equality holds in virtue of \lemref{lem:generalized-orb-aut}. The first equality follows from the two considerations mentioned: (1) dividing by the number of backward orbit-equivalent actions to account for duplicate parents; (2) multiplying inflows by the number of forward orbit-equivalent actions to account for duplicate actions.
    On the other hand, the outflows are directly computed as:
    \[
        \sum_{s'':(s', s'')\in\gA} F(s' \rightarrow s'') = \sum_{G'':(G',G'')\in\gE} \tilde F(G'\rightarrow G'').
    \]
    Rearranging terms, we have the result.
\end{proof}

\section{Discussion on the Fragment-based Generation}
\label{app:fragment-based}

\begin{figure}[ht]
\centering
    \begin{tikzpicture}[
    roundnode/.style={circle, draw, minimum size=3mm, inner sep=0pt}
]


\node[roundnode, fill=red!30] (g1n1) at (0, 0.5) {};
\node[font=\small] at (0, 1.6) {$|\Aut(G)| = 1$};
\node at (0, -0.8) {$G$};
\node (g1r1) at (0.6, -0.9) {};
\node (g1r2) at (0.6, -0.7) {};

\begin{scope}[xshift=4cm]

\node[roundnode, fill=red!30] (g2n1) at (-1, 0.5) {};

\node[roundnode, fill=yellow!30] (g2n1) at (1, 0) {};
\node[roundnode, fill=yellow!30] (g2n2) at (1, 1) {};
\node[roundnode, fill=yellow!30] (g2n3) at (1-0.866, 0.5) {};
\node[font=\small] at (0, 1.6) {$|\Aut(G')| = 6$};
\node at (0, -0.8) {$G'$};
\node (g2l1) at (-0.6, -0.9) {};
\node (g2l2) at (-0.6, -0.7) {};

\draw (g2n1) -- (g2n2);
\draw (g2n2) -- (g2n3);
\draw (g2n3) -- (g2n1);

\end{scope}

\draw[{Latex[length=2mm]}-, thick] (g1r1) -- (g2l1) node[midway, below, font=\tiny] 
    {$|\Orb(G, C)|=1$};
\draw[-{Latex[length=2mm]}, thick] (g1r2) -- (g2l2) node[midway, above, font=\tiny] 
    {$|\Orb(G', V)|=1$};
\end{tikzpicture}
    \caption{Transition representing the \texttt{AddFragment} action.}
    \label{fig:fragment-example}
\end{figure}
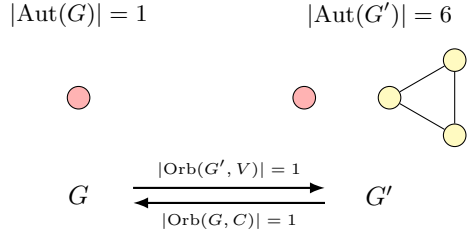

In the case the action \texttt{AddFragment(G, C)}, which adds a fragment $C$ to the existing graph $G$ resulting in $G' = G \cup C$, we must account for the additional symmetries introduced by the fragment. In \figref{fig:fragment-example}, there is only one way of adding and deleting the fragment from $G$ to $G'$, while the fragment induces 6 symmetries in $G'$.

We begin by defining the notion of orbits for actions \texttt{AddFragment(G, C)} and \texttt{RemoveFragment(G', C)}. The orbit of \texttt{AddFragment(G, C)} is defined as \(\Orb(G, V)\), whose cardinality is 1. This is because each fragment $C$ in the vocabulary leads to a unique next graph when added to $G$, so the number of forward transition-equivalent actions is 1. For the corresponding backward action \texttt{RemoveFragment(G', C)}, however, there may be multiple subgraphs of $G'$ that are isomorphic to $C$. Hence, we define the orbit as \texttt{RemoveFragment(G', C)} is $\Orb(G', C) = \{C' : \exists \pi \in \Aut(G'), \pi(C) = C'\}$. This set captures all subgraphs of $G'$ that are automorphic images of $C$, i.e., all valid candidates for removal that are equivalent under the symmetry of $G'$.

Next, we extend \lemref{lem:orb-aut} to accommodate the fragment-level actions, which account for the symmetries of both the existing graph and the fragment.

\begin{lemma}
\label{app:orb-aut-frag}
Let $G = (V_G, E_G)$ be a graph representing the current state. We consider augmenting the graph $G$ by adding a fragment $C = (V_C, E_C)$. Let $G \cup C = (V_G \cup V_C, E_G \cup E_C)$ denote the union of the two graphs (without any edges connecting $G$ and $C$). Then, we have:
$$
\frac{\text{Number of forward orbit-equivalent actions}}{\text{Number of backward orbit-equivalent actions}}
=\frac{|\Orb(G, V)|}{|\Orb(G \cup C, C)|}
= \frac{|\Aut(G)| \cdot |\Aut(C)|}{|\Aut(G \cup C)|}.
$$
\end{lemma}
\begin{proof}
    Since $|\Orb(G, V)|=1$, we only need to consider $|\Orb(G \cup C, C)|$. The stabilizer $\Stab(G \cup C, C)$ is the set of automorphisms in $\Aut(G \cup C)$ that does not mix the labels of $G$ and $C$; it acts independently on $G$ and $C$. Therefore, the order of $\Stab(G \cup C, C)$ is $|\Aut(G)| \cdot |\Aut(C)|$. Using the orbit-stabilizer theorem, we obtain:
    \begin{equation*}
    \begin{aligned}
        |\Orb(G\cup C, C)|
        &= \frac{|\Aut(G\cup C)|}{|\Stab(G\cup C, C)|} \\
        &= \frac{|\Aut(G\cup C)|}{|\Aut(G)|\cdot|\Aut(C)|}.
    \end{aligned}
    \end{equation*}
\end{proof}

Using \lemref{app:orb-aut-frag}, we obtain fragment correction formula in \thmref{thm:Fragment correction}. 

\begin{theorem}[Fragment correction]
\label{app:frag-correction}
Let $G$ represents a terminal state ($[G]\in\gX$) generated by connecting k fragments $\{C_1, \dots, C_k\}$. Then, the scaled rewards to offset the effects of equivalent actions are given by:
$$
\tilde R(G) = \frac{|\Aut(G)|R(G)}{\prod_{i=1}^k|\Aut(C_i)|}.
$$
\end{theorem}
\begin{proof}
    Using \lemref{app:orb-aut-frag}, we first derive similar results for \thmref{thm:Automorphism correction}. If $p_\gE$ and $q_\gE$ are permutation-equivariant functions, we have
    \[
        \frac{p_{\bar\gA}(a|s)}{q_{\bar\gA}(a|s')} = 
            \frac{|\Aut(G)|\cdot|\Aut(C)|}{|\Aut(G')|} \cdot \frac{p_\gE(G'|G)}{q_\gE(G|G')}.
    \]
    We defined \texttt{AddFragment($G, C$)} as adding a fragment $C$, resulting in disconnected graph $G\cup C$. To connect between fragments, we use \texttt{AddEdge} and \lemref{lem:orb-aut:AddEdge}. Then, the results follows from the TB objective written in terms of graph transitions $p_\gE$ and $q_\gE$, and the telescoping sum, as in \corref{cor:TB-correction}.
\end{proof}

Unlike atom-based generation, the fragment terms $|\Aut(C)|$ do not cancel out through a telescoping sum. Therefore, these terms must be explicitly accounted for in the correction, both for reward scaling and estimating the model likelihood.

\begin{figure}[t]
\centering
    \begin{tikzpicture}[
    roundnode/.style={circle, draw, minimum size=3mm, inner sep=0pt}
]
\def\radius{0.6}
\def\highlightsize{1.5}

\draw[rotate=30] 
    (90:\radius) -- (30:\radius) -- (-30:\radius) -- (-90:\radius) -- (-150:\radius) -- (150:\radius) -- cycle;
\fill[orange!50, rotate=30] (90:\radius) circle(\highlightsize mm);
\fill[orange!50, rotate=30] (150:\radius) circle(\highlightsize mm);
\fill[orange!50, rotate=30] (-30:\radius) circle(\highlightsize mm);

\node (leftanchor) at (\radius+0.5, 0) {};

\begin{scope}[xshift=3.6cm, yshift=1.5cm]

\draw[rotate=30] 
    (90:\radius) -- (30:\radius) -- (-30:\radius) -- (-90:\radius) -- (-150:\radius) -- (150:\radius) -- cycle;
\draw[rotate=30] (-30:\radius) -- (-30:2*\radius);
\fill[orange!50, rotate=30] (-30:\radius) circle(\highlightsize mm);
\fill[orange!50, rotate=30] (30:\radius) circle(\highlightsize mm);
\fill[orange!50, rotate=30] (-150:\radius) circle(\highlightsize mm);

\node (middleanchor1) at (2*\radius+0.2, 0) {};
\end{scope}

\begin{scope}[xshift=3.6cm]

\draw[rotate=30] 
    (90:\radius) -- (30:\radius) -- (-30:\radius) -- (-90:\radius) -- (-150:\radius) -- (150:\radius) -- cycle;
\draw[rotate=30] (-30:\radius) -- (-30:2*\radius);
\fill[orange!50, rotate=30] (90:\radius) circle(\highlightsize mm);
\fill[orange!50, rotate=30] (150:\radius) circle(\highlightsize mm);
\fill[orange!50, rotate=30] (-30:\radius) circle(\highlightsize mm);
\node (middleanchor2left) at (-\radius-0.5, 0) {};
\node (middleanchor2) at (2*\radius+0.2, 0) {};
\end{scope}

\begin{scope}[xshift=3.6cm, yshift=-1.5cm]

\draw[rotate=30] 
    (90:\radius) -- (30:\radius) -- (-30:\radius) -- (-90:\radius) -- (-150:\radius) -- (150:\radius) -- cycle;
\draw[rotate=30] (-30:\radius) -- (-30:2*\radius);
\fill[orange!50, rotate=30] (-90:\radius) circle(\highlightsize mm);
\fill[orange!50, rotate=30] (-30:\radius) circle(\highlightsize mm);
\fill[orange!50, rotate=30] (-210:\radius) circle(\highlightsize mm);
\node (middleanchor3) at (2*\radius+0.2, 0) {};
\end{scope}

\begin{scope}[xshift=8cm]
\draw[rotate=30] 
    (90:\radius) -- (30:\radius) -- (-30:\radius) -- (-90:\radius) -- (-150:\radius) -- (150:\radius) -- cycle;
\draw[rotate=30] (-30:\radius) -- (-30:2*\radius);
\node (rightanchorleft) at (-\radius-0.5, 0) {};
\end{scope}

\draw[-{Latex[length=2mm]}, thick] (leftanchor) -- (middleanchor2left);
\draw[-{Latex[length=2mm]}, thick] (middleanchor1) -- (rightanchorleft);
\draw[-{Latex[length=2mm]}, thick] (middleanchor2) -- (rightanchorleft);
\draw[-{Latex[length=2mm]}, thick] (middleanchor3) -- (rightanchorleft);

\end{tikzpicture}
    \caption{A fragment with attachment points highlighted. Attachment points are designed such that they break symmetries of the fragment. Rightmost graph represent the terminal state where attachment points are removed.}
    \label{fig:fragment-failure}
\end{figure}
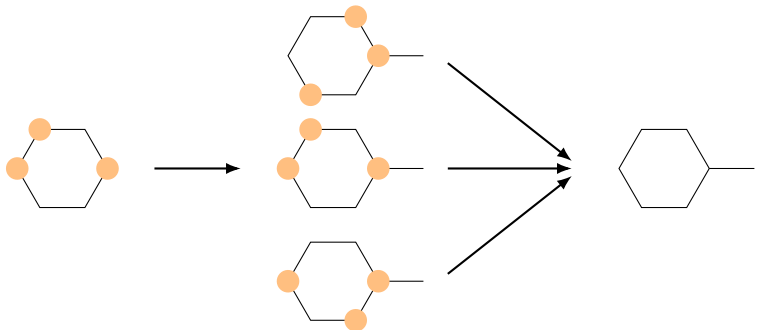

We need to exercise caution when applying \lemref{app:frag-correction}, as the validity of the result depends on the specifics of the action design. A common practice in the GFlowNet literature is to predefine a set of attachment points for each fragment, where attachment points refer to nodes where new edges can connect to other fragments. These attachment points should be treated as node attributes, even if they are artifacts of the generation process rather than intrinsic properties of the graph. The reason for this treatment is that attachment points constrain the set of allowable actions, including the set of equivalent actions. For example, even if two nodes $u$ and $v$ belong to the same orbit, they should be considered distinct if only one of them is marked as an attachment point.

The complication arises when terminal states are considered. Since attachment points are not intrinsic properties of the graph, conceptually, a graph receives its reward after the attachment points are removed as shown in \figref{fig:fragment-failure}. However, this removal introduces three distinct backward actions from the terminal state in the figure, which complicates the calculation of the backward probabilities. 

This issue does not arise, however, if we arrange attachment points in a fragment such that nodes in different orbits (i.e., orbits that consider attachment points as node attributes) remain different even after the attachment points not present. We observe that this holds for the fragments used in \citet{bengio2021flow}.

\section{Relation to Node Orderings}
\label{appendix:node-ordering}

Some previous work on graph generation uses a distribution over permutations (or node orderings) $\pi$, treating it as a random variable \citep{chen2021order, kong2023autoregressive,wang2025learning}. If the node (or edge) ordering $\pi$ is given, the generation path $s_{0:n}=(s_0, \dots, s_n)$ can be determined. However, the converse is not generally true: for each state sequence $s_{0:n}$, there can be multiple compatible orderings $\pi$, complicating the computation of variational lower bound. 

\citet{chen2021order} derived an exact formula for the number of node orderings consistent with a given trajectory $s_{0:n}$: it is given by the product of the sizes of the orbits encountered during the generative process, $\prod_{t=0}^n|\Orb(s_t,a_t)|$. To make this computation tractable, they approximate the orbit sizes using the color refinement algorithm. This estimate is then used to compute the joint probability $p(s_{0:n}, \pi)$.

However, the joint probability can be computed easily if we consider the corresponding graph sequence $G_{0:n}$, as $p(s_{0:n}, \pi) = \prod _{t=0}^{n-1}p_\gE(G_{t+1}|G_t)$. This holds because the number of different orderings that induce the same state sequence $s_{0:n}$ corresponds to the number of distinct paths generated by following transition-equivalent actions. In this view, different actions that are transition-equivalent can be interpreted as actions that induce different orderings while resulting in the same sequence of graph states. See \citet{chen2021order} for more details on using node orderings as a random variable.

\section{Computational Cost}
\label{app:Computational Cost}

\subsection{Computation Time for Counting Automorphisms}
While computing the exact \( |\text{Aut}(G)| \) has inherent complexity, this complexity is unavoidable for exact GFlowNets. In practice, fast heuristic algorithms computing \( |\text{Aut}(G)| \) often perform well, particularly for relatively small graphs. We provide computation time of \( |\text{Aut}(G)| \) for several molecular dataset.

\begin{table}[h]
\caption{Computational cost. ``Large'' dataset refers to the largest molecules in PubChem, which is used in the paper \citet{flam2022language}. Experiments were conducted on an Apple M1 processor.}
\label{tab:compute-time}
\begin{center}
\begin{tabular}{lllll}
\hline
\bf Dataset  &\bf Sample Size &\bf Num Atoms &\bf Compute time (\textit{bliss}) &\bf Compute time (\textit{nauty}) \\
\hline
QM9 & 133,885 & 8.8 ± 0.5 & 0.010 ms ± 0.008 & 0.019 ms ± 0.079\\
ZINC250k & 249,455 & 23.2 ± 4.5 & 0.022 ms ± 0.010 & 0.042 ms ± 0.032 \\
CEP & 29,978 & 27.7 ± 3.4 & 0.025 ms ± 0.014 & 0.050 ms ± 0.076\\
Large & 304,414 & 140.1 ± 49.4 & - & 0.483 ms ± 12.600 \\
\hline
\end{tabular}
\end{center}
\end{table}

Compared to sampling trajectories, which involves multiple forward passes through a neural network, the compute time for \( |\Aut(G)| \) is negligible. For comparison, we report the speed of molecular parsing algorithms measured using ZINC250k dataset: 0.06 ms ± 0.70 (SMILES → molecule) and 0.04 ms ± 0.05 (molecule → SMILES). The combination of two parsing steps is often used to check the validity of a given molecule in various prior works. In words, computing \( |\Aut(G)| \) is in an order of magnitude faster than validity checking algorithm.

We used the \textit{bliss} algorithm for our experiment. It is easy to use as it is included in the \texttt{igraph} package and is fast enough for our purposes. For large molecules, we can still count automorphisms in few milliseconds using the \textit{nauty} package \citep{mckay2013nauty} as can be seen in the table. We observed that the \texttt{pynauty} package does not natively support distinguishing between different edge types, requiring us to transform the input graphs by attaching virtual nodes to handle this. The reported time in the table reflects these preprocessing steps.

While we believe the computation time is already negligible for current applications, we outline two additional strategies to further reduce runtime:
\begin{enumerate}
    \item \textbf{Parallelization.} Data processing tasks can be parallelized across multiple CPUs. Since GFlowNet is an off-policy algorithm, \( |\Aut(G)| \) can be computed concurrently with the policy learning.
    \item \textbf{Approximate correction for large graphs.} For large graphs, fragment-based generation is highly likely to be employed. In such cases, we can apply an approximate correction scheme, as outlined in \appref{app:fragment-based}.
\end{enumerate}

\subsection{Training Time Comparison}

To evaluate the training time, we ran three separate training sessions for each method. \tabref{tab:wallclock-runtime} reports the total training time for each method in the synthetic environment.

\begin{table}[h]
\centering
\begin{tabular}{lccc}
\toprule
\textbf{Method} & \textbf{1000 Steps (s)} & \textbf{3000 Steps (s)} & \textbf{5000 Steps (s)} \\
\midrule
Transition Correction   & 1338 ± 79 & 4122 ± 168 & 6859 ± 80 \\
PE (Ma et al., 2024)    & 1322 ± 44 & 4001 ± 131 & 6604 ± 152 \\
Reward Scaling (Ours)   & 1178 ± 31 & 3584 ± 97  & 5999 ± 146 \\
Flow Scaling (Ours)     & 1176 ± 32 & 3577 ± 100 & 5987 ± 168 \\
\bottomrule
\end{tabular}
\caption{Wall-clock runtime (in seconds) for different methods over increasing training steps. Mean and standard deviation are computed over 3 runs.}
\label{tab:wallclock-runtime}
\end{table}

While PE is faster than Transition Correction---which performs multiple isomorphism tests per transition---our proposed methods, \emph{Reward Scaling} and \emph{Flow Scaling}, are even more efficient. All timing experiments were conducted using a single processor with a TITAN RTX GPU (24GB) and an Intel Xeon Silver 4216 CPU. These results demonstrate that our methods incur significantly less computational overhead while preserving correctness.

\subsection{Scalability Comparison}

\begin{wrapfigure}{tr}{0.45\textwidth}
    \vspace{-30pt}
    \centering
    \includegraphics[width=0.45\textwidth]{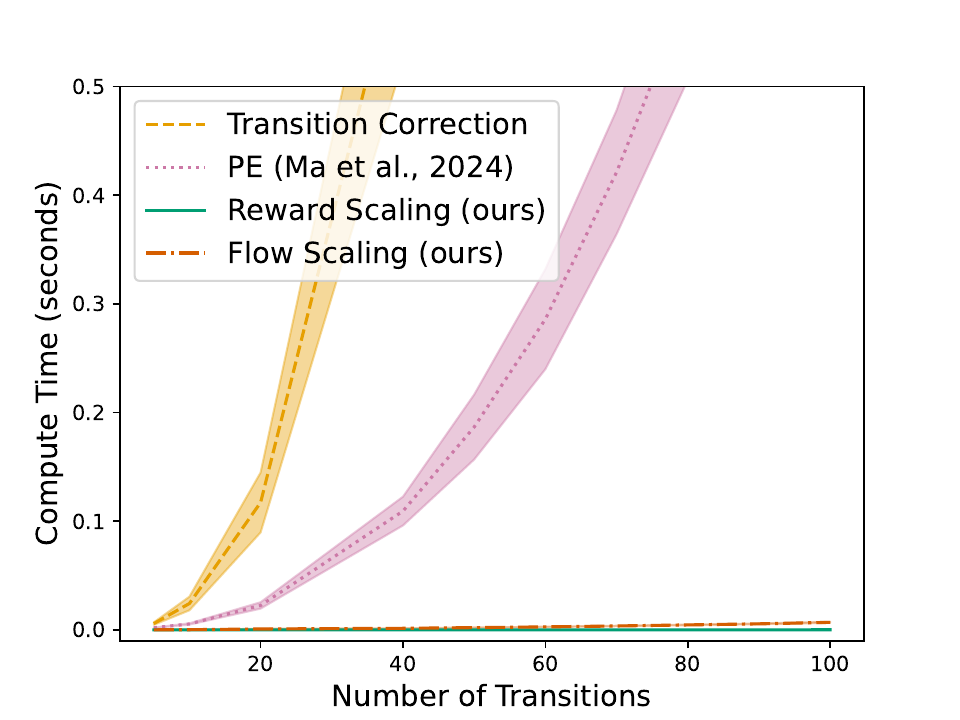}
    \caption{}
    \label{fig:uniform-GNN-failure}
    \vspace{-20pt}
\end{wrapfigure}

We evaluated the computation time for each method, varying the number of transitions per trajectory. We sampled 100 random graphs for each horizon length category. We measured the time spent on only the major components of each method. This includes:  
1) multiple isomorphism tests for \textbf{Transition Correction};  
2) computing positional encodings and matching encodings to identify isomorphic graphs for \textbf{PE};
3) computing automorphisms of the final graph for \textbf{Reward Scaling};  
4) computing automorphisms of all intermediate graphs for \textbf{Flow Scaling}.  

When computing PEs, we used $k=8$. Note that methods 1) and 2) must be applied for both forward and backward actions. The table below reports the additional computational cost incurred by each method per trajectory, where Vanilla GFlowNets is considered to incur zero additional cost.

\begin{table}[th]
\centering
\begin{tabular}{lccc}
\toprule
\textbf{Method} & \textbf{10 Transitions (ms)} & \textbf{50 Transitions (ms)} & \textbf{100 Transitions (ms)} \\
\midrule
Transition Correction & 24.32 ± 6.28 & 1148 ± 240.3 & 7354 ± 1288 \\
PE (Ma et al., 2024)  & 5.49 ± 0.58 & 186.7 ± 29.87 & 997.5 ± 133.9 \\
Reward Scaling (Ours) & 0.024 ± 0.002 & 0.063 ± 0.004 & 0.111 ± 0.008 \\
Flow Scaling (Ours)   & 0.215 ± 0.025 & 2.106 ± 0.116 & 6.975 ± 0.421 \\
\bottomrule
\end{tabular}
\caption{Runtime comparison for computing correction terms under different methods.}
\label{tab:runtime-comparison}
\end{table}

While the cost increases for all methods as the number of transitions grows, our method clearly scales better. It is important to note that the time differences accumulate over the entire training duration. Experiments were performed using Intel Xeon Silver 4216 CPU.

\section{Experimental Details}
\label{app:experimental-details}

\subsection{Graphs-building Environments}

For graph-building environments, both for illustrative example and synthetic graphs, we stacked 5 GPS layers with 256 embedding dimensions \citep{rampavsek2022recipe}. To increase representation power, we augmented node features to increase representation power. We augmented node features with one-hot node degree, clustering coefficient, and 8 dimensions of random-walk positional encoding \citep{dwivedi2021graph}. Importantly, edge features were computed by summing node features outputted from GNN layers, which, in turn, concatenated with shortest path length between two nodes connected by the edge. This is to prevent representation collapse between different orbits (see \appref{app:Expressive Power of GNNs}). 

\paragraph{Illustrative Example.} For the illustrative experiment, homogeneous graphs were constructed edge by edge, allowing only \texttt{AddEdge} and \texttt{Stop} actions. The initial state consisted of six disconnected nodes, and terminal states corresponded to connected graphs without isolated nodes. The number of terminal states, $|\gX|$, is $112$.

We trained the models for 30,000 updates using the TB objective. During the first 16,000 steps, each update used a batch of 128 trajectories, comprising 32 samples from the current policy and 96 samples drawn from the replay buffer. We used the Adam optimizer \citep{kingma2014adam} with the default parameters from \texttt{PyTorch} \citep{NEURIPS2019_9015} settings, except for the learning rates: 0.0001 for GNN layers and 0.01 for the normalizing constant $Z$. For the remaining steps, we increased batch size to 256 and annealed the learning rate to 0.00001. To encourage exploration, the policy selected actions uniformly with a probability of 0.9.

\paragraph{Synthetic Graphs.} For the Synthetic Graphs experiment, we followed the environment setup described in \citep{ma2024baking}. Each node can be one of two types, and graphs can contain up to 7 nodes, resulting in a total of 72,296 terminal states. The generation process starts from the empty graph, from which the policy incrementally adds nodes and edges. The reward function was also adopted from \citep{ma2024baking}. 

Models were updated 50,000 times using both the TB and DB objectives. For the TB objective, 64 trajectories were used for updates, with 16 sampled from the behavior policy. For the DB objective, 16 trajectories were sampled per step and converted into transitions, with each update utilizing 512 transitions. We used the Adam optimizer with a learning rate 0.0001 for GNN layers and 0.01 for the normalizing constant $Z$.

\paragraph{PE Implementation Details.} Following \citet{ma2024baking}, we implemented the positional embedding (PE) method as follows: First, the random walk matrix is computed as $(AD^{-1})^k$, where $A$ is the adjacency matrix, and $D$ is the degree matrix with node degrees on the diagonal. Next, we multiply the matrix by a vector $c$ representing node types. We constructed the vector $c$ by setting each element to $\log(t + 2)$, where $t \in \{0, 1, \dots\}$ denotes the node type ID. By varying $k$ from 0 to 7, we obtained 8-dimensional PEs for each node. Edge-level PEs were then computed by summing the embeddings of the corresponding node pairs. Finally, actions sampled from the policy were compared with other candidate actions based on their PEs to identify equivalence.

\subsection{Molecule Generation}
\label{app:experimental-details:Molecule Generation}
We conducted experiments on small molecule generation tasks following \cite{bengio2021flow, jain2023multi}. More detailed task descriptions can be found in these previous works. We used a open-source code for tasks.\footnote{\url{https://github.com/recursionpharma/gflownet}} We used a graph transformer architecture \citep{yun2019graph} with the hyperparameters summarized in \tabref{tab:hyperparameters-atom} and \tabref{tab:hyperparameters-fragment}. In GFlowNets, the reward exponent $\beta$ is used to focus sampling on high-reward regions in the state space. The correction is applied after rewards are exponentiated: $C(x)R(x)^{\beta}$, where $C(x)$ is the correction term.

\begin{table}[h]
\caption{Hyperparameters for atom-based experiments}
\label{tab:hyperparameters-atom}
\begin{center}
\begin{tabular}{lll}
\hline
& \bf Hyperparameters  &\bf Values \\
\hline
\multirow{7}{4em}{Training} & 
    Learning Rate ($p_\gE, Z$) &0.0005\\
    & Batch Size (Online)     &32 \\
    & Batch Size (Buffer)     &32 \\
    & Uniform Exploration $\epsilon$  &0.1 \\
    & Gradient Clipping (Layer-wise Norm) &10.0 \\
    & Reward Exponent $\beta$        &1 \\
    & Number of Updates  & 30,000 \\
\hline
\multirow{5}{4em}{Model} & 
    Architecture          & Graph Transformer  \\
    & Number of Layers     & 4 \\
    & Number of Heads     & 4 \\
    & Number of Embeddings  & 128 \\
    & Number of Final MLP Layers  & 1 \\
\hline
\end{tabular}
\end{center}
\end{table}

\begin{table}[h]
\caption{Hyperparameters for fragment-based experiments}
\label{tab:hyperparameters-fragment}
\begin{center}
\begin{tabular}{lll}
\hline
& \bf Hyperparameters  &\bf Values \\
\hline
\multirow{8}{4em}{Training} & 
    Learning Rate ($p_\gE$) &0.0001\\
    & Learning Rate ($Z$) &0.001\\
    & Batch Size (Online)     &32 \\
    & Batch Size (Buffer)     &32 \\
    & Exploration $\epsilon$  &0.1 \\
    & Gradient Clipping (Layer-wise Norm) &10.0 \\
    & Reward Exponent $\beta$        &16 \\
    & Number of Updates  & 30,000 \\
\hline
\multirow{5}{4em}{Model} & 
    Architecture          & Graph Transformer  \\
    & Number of Layers     & 5 \\
    & Number of Heads     & 4 \\
    & Number of Embeddings  & 256 \\
    & Number of Final MLP Layers  & 2 \\
\hline
\end{tabular}
\end{center}
\end{table}

\paragraph{Evaluation Metrics.} Evaluation metrics we presented in \tabref{tab:mol-results} are defined as follows:

\begin{itemize}
    \item \textbf{Diversity.} The average pairwise Tanimoto distance between molecules sampled from the trained policy.
    \item \textbf{Top $K$ diverse.} Diversity among the top $K$ reward molecules.
    \item \textbf{Top $K$ reward.} The average reward of the top $K$ molecules.
    \item \textbf{Diverse top $K$.} The average reward of the top $K$ molecules, ensuring that each pair has a Tanimoto distance greater than $0.7$.
    \item \textbf{Unique fraction.} The fraction of unique molecules in the generated samples.
\end{itemize}

We selected $K=50$, which corresponds to the top 10\% of molecules for our evaluation. When reporting rewards, we adjust them to remove the effects of reward scaling and reward exponents.

For the Pearson correlation evaluation presented in \figref{fig:mol-pearson}, terminal states were sampled by uniformly selecting random actions. The model likelihood was computed using \eqnref{eqn:model-likelihood}, with a modified correction term in \thmref{thm:Fragment correction}. We set $M=5$ and used 2,048 samples for the test set.

\subsection{Fragment Correction Method}

For fragment-based molecule generation, we used a predefined set of fragments and attachment points provided by \citet{bengio2021flow}. There are a total of 72 fragments, each with a varying number of attachment points. Our method requires pre-computing the number of automorphisms for each fragment. In \figref{fig:fragments72}, we present the number of automorphisms for each fragment used in our experiment. As discussed in \appref{app:fragment-based}, attachment points were treated as distinct attributes when counting automorphisms.

\begin{figure}[h]
\centering
    \includegraphics[width=1.0\textwidth]{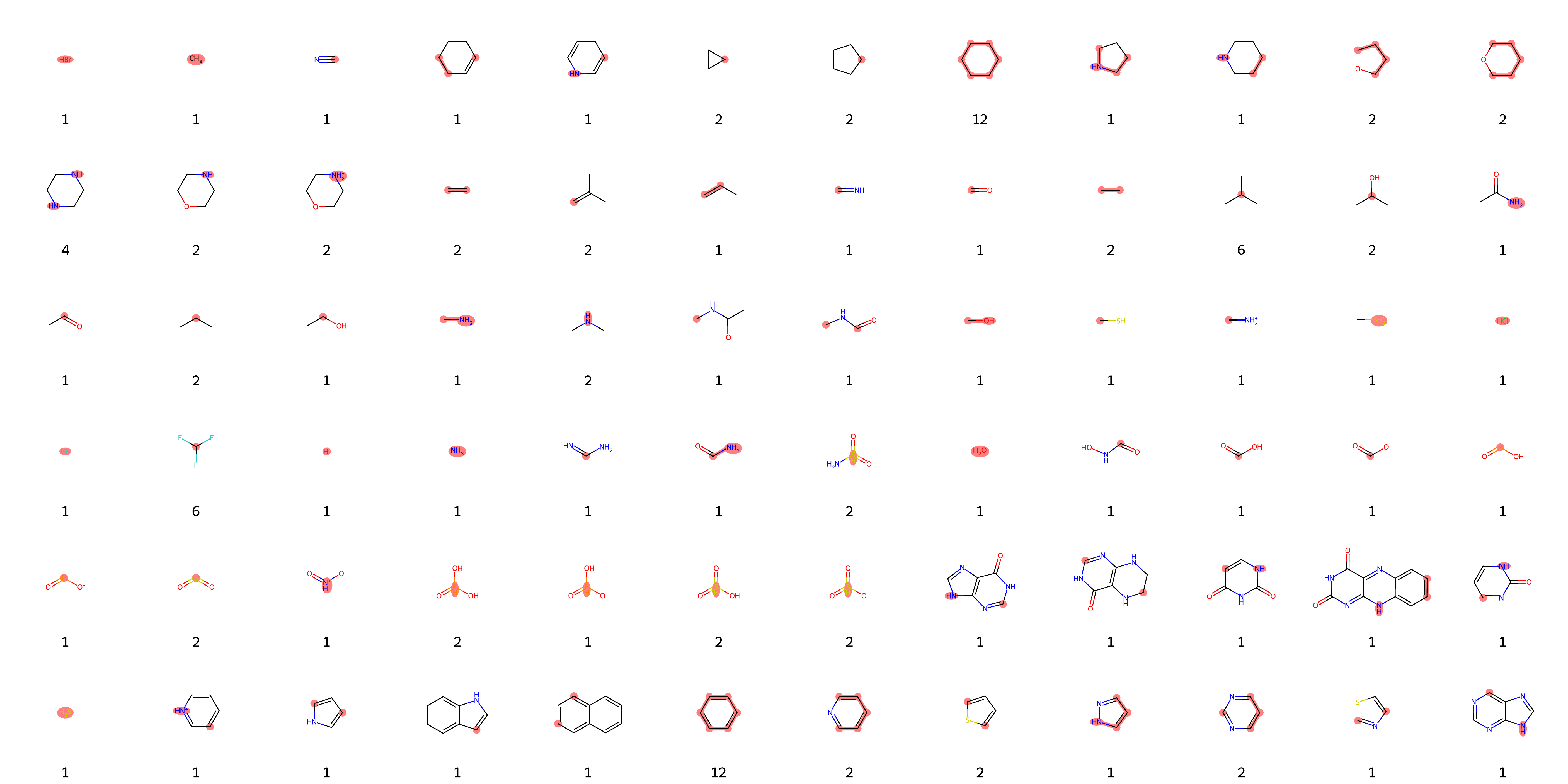}
    \caption{Predefined fragment set used for the fragment-based task. Attachment points, where a single bond can connect to another fragment, are highlighted in red. The numbers indicate the order of the automorphism group of each fragment, $|\Aut(C)|$.}
\label{fig:fragments72}
\end{figure}

\subsection{Approximate Correction Method}
\label{app:Approximate Correction Method}

Additionally, we experimented with a simplified version where the correction is applied approximately for the fragment-based task. While we can compute exact correction term as in \eqnref{eqn:frag-correction}, this approximation provides computational benefits, as it avoids counting automorphisms. Moreover, similar approximations can be easily implemented even for more complex generation schemes that do not fit into \eqnref{eqn:frag-correction}. The approximation works as follows: we assign a number to each fragment based on how many equivalent actions it is likely to incur during generation. We adjust the final rewards by dividing them by the product of the assigned numbers $N$ for the constituent fragments: $R(G)/\prod_{i=1}^kN(C_i)$.

We assigned the number $N$ to each fragment based on how likely it is to incur forward equivalent actions. This is because fragments that incur multiple forward equivalent actions are more likely to be selected if no adjustment is applied. For example, cyclohexane (\texttt{C1CCCCC1}) has six attachment points, all in the same orbit, so it will always incur at least six forward equivalent actions in subsequent steps. In contrast, even if a fragment is highly symmetric, if it has only one attachment point, it will incur no equivalent actions. We assigned $N=1$ to such fragments.

Assuming backward equivalent actions are relatively rare in fragment-based generation, this approximation should closely match the unbiased correction. These numbers were assigned through visual inspection of the fragments. The full set of fragments and their assigned numbers for the approximate correction is provided in \figref{fig:approx-fragments72}.

\begin{figure}[h]
\centering
    \includegraphics[width=1.0\textwidth]{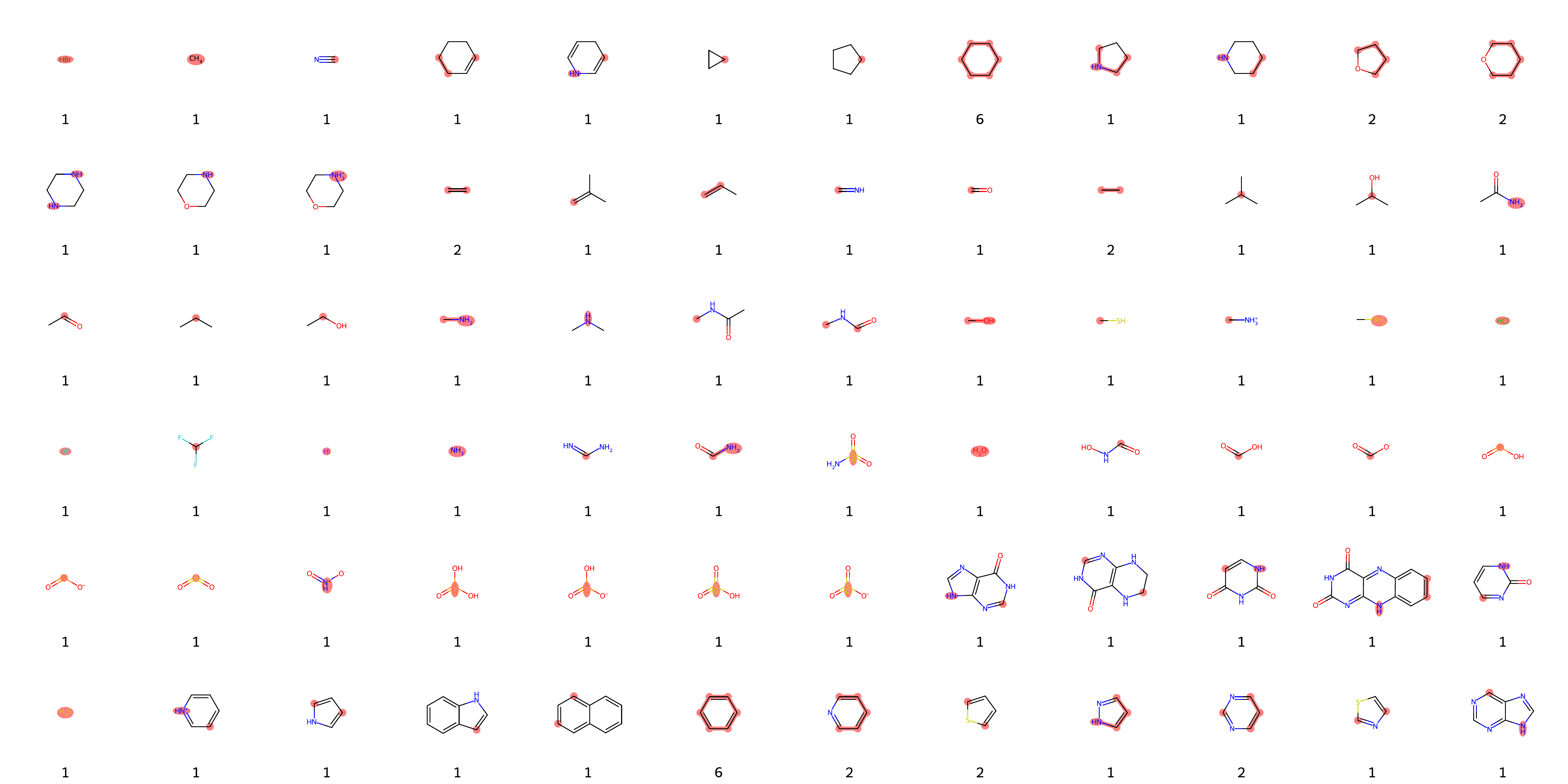}
    \caption{Predefined fragment set used for the fragment-based task. Attachment points are highlighted in red. The numbers below each molecule are used for approximate correction.}
    \label{fig:approx-fragments72}
\end{figure}

\section{Additional Experimental Results}
\label{app:Additional-Experimental-Results}

\subsection{Synthetic Graphs}

\begin{figure*}[t]
\begin{subfigure}{0.45\textwidth}
  \centering
  \includegraphics[width=\linewidth]{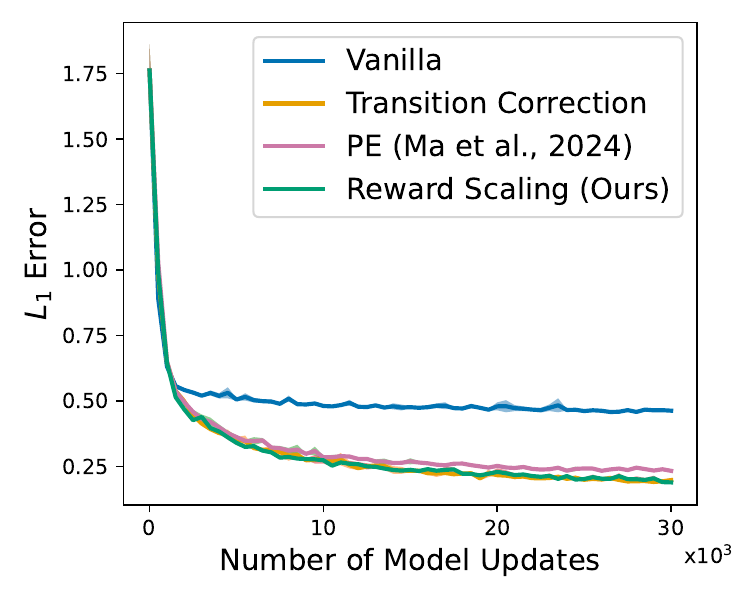}
  \caption{Synthetic (TB)}
  \label{appfig:sfig1}
\end{subfigure}
\hfill
\begin{subfigure}{0.45\textwidth}
  \centering
  \includegraphics[width=\linewidth]{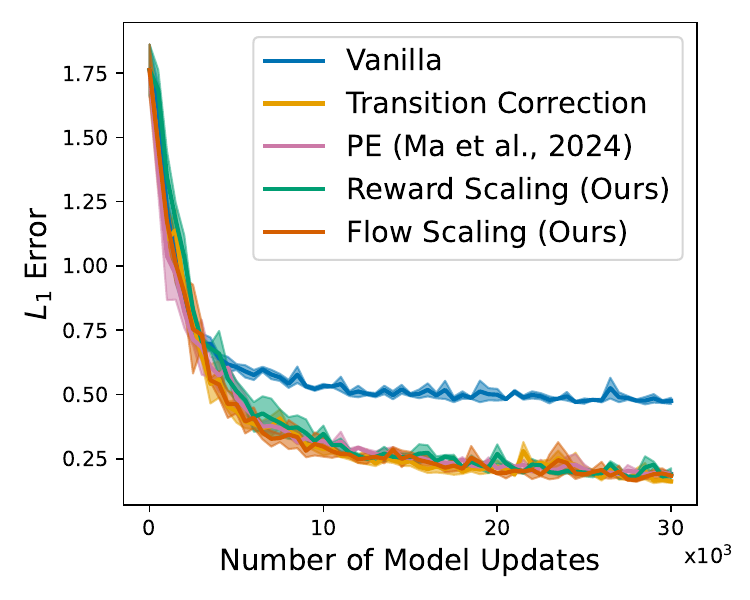}
  \caption{Synthetic (DB)}
  \label{appfig:sfig2}
\end{subfigure}
\caption{Results on synthetic experiments, where rewards are given based on the number of cycles in the graph.}
\label{appfig:synthetic results}
\vspace{-10pt}
\end{figure*}

In addition to the clique environment presented in the main text, we conducted similar experiments under a different reward structure. We restricted the maximum number of nodes and edges to 10, and limited the maximum node degree to 4, thereby constraining the state space to $|\gX| = 2,999$. Rewards were defined as 1 + (number of cycles). The hyperparameters are not tuned and remain the same as those explained in \appref{app:experimental-details}.

For the TB objective shown in \figref{appfig:synthetic results}~(a), the PE method slightly underperforms compared to Reward Scaling, which is consistent with the results in the clique environment. In contrast, for the DB objective shown in \figref{appfig:synthetic results}~(b), where intermediate probabilities play a crucial role, the PE method performs comparably to Reward Scaling. In all cases, Transition Correction and Flow Scaling perform the best.

\subsection{Molecule Generation}

We investigated which fragments were sampled by each method. We sampled 5,000 terminal states from trained models, resulting in 44,974 and 44,978 fragments sampled from vanilla and corrected model, respectively. Symmetric fragments were found to be sampled more frequently in vanilla model, which aligns with our projection, as the fragment correction in \eqnref{eqn:frag-correction} penalizes symmetric components. However, the proportions of fragments between the two methods are not exactly proportional to the magnitude of the corrections, as some fragments are more likely to occur together (they are not independent).

\begin{figure}[h]
\centering
    \includegraphics[width=0.8\textwidth,trim=4 4 4 4,clip]{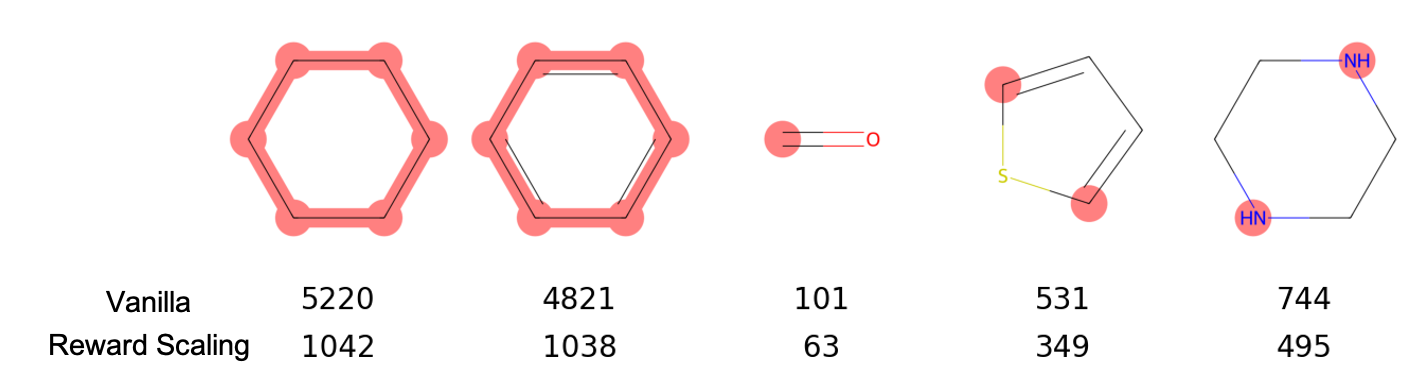}
    \caption{The number of sampled fragments from 5,000 terminal states for vanilla model and corrected model. We display the 5 fragments that were sampled most disproportionately. Attachment points are highlighted in red.}
\label{fig:sym_frags}
\end{figure}

\end{document}